\documentclass[sigconf]{acmart}

\renewcommand\footnotetextcopyrightpermission[1]{} 
\usepackage{balance} 

\usepackage[utf8]{inputenc}

\usepackage{booktabs} 
\usepackage{pgfplots}

\copyrightyear{2020}
\acmYear{2020}
\setcopyright{acmlicensed}\acmConference[GECCO '20]{Genetic and Evolutionary Computation Conference}{July 8--12, 2020}{Cancún, Mexico}
\acmBooktitle{Genetic and Evolutionary Computation Conference (GECCO '20), July 8--12, 2020, Cancún, Mexico}
\acmPrice{15.00}
\acmDOI{10.1145/3377930.3390218}
\acmISBN{978-1-4503-7128-5/20/07}

\usepackage{graphicx}
\usepackage{amsmath,amssymb,amsthm,mathtools}
\usepackage{bm}
\usepackage{lipsum}
\usepackage{xspace}
\usepackage{algorithm}
\usepackage[noend]{algorithmic}
\usepackage{enumerate}

\DeclareMathOperator{\prob}{Pr}
\DeclareMathOperator{\E}{E}

\newcommand{\EA}{\text{(1+1)~EA}\xspace}
\newcommand{\EAc}[1][\chi]{(1+1)$_{{#1}}$~EA\xspace}
\newcommand{\wop}{w.\,o.\,p.\xspace}

\graphicspath{ {./images/} }
\author{George T. Hall}
\affiliation{%
  \institution{Department of Computer Science}
  \city{University of Sheffield, Sheffield, UK}
}

\author{Pietro S. Oliveto}
\affiliation{%
  \institution{Department of Computer Science}
  \city{University of Sheffield, Sheffield, UK}
}

\author{Dirk Sudholt}
\affiliation{%
  \institution{Department of Computer Science}
  \city{University of Sheffield, Sheffield, UK}
}

\title[Analysis of Algorithm Configurators for Search Heuristics with Global Mutation]{Analysis of the Performance of Algorithm Configurators for Search Heuristics with Global Mutation Operators}

\thanks{A version of this paper without the appendix is to appear at GECCO 2020.}

\begin{document}

\begin{abstract}

    Recently it has been proved that a simple algorithm
    configurator called ParamRLS can efficiently identify the optimal
    neighbourhood size to be used by stochastic local search to optimise two
    standard benchmark problem classes. In this paper we analyse the
    performance of 
    algorithm configurators for tuning the more sophisticated global mutation
    operator used in standard evolutionary algorithms, which flips each of the
    $n$ bits independently with probability $\chi/n$ and the best value for $\chi$
    has to be identified.  We compare the performance of 
    configurators when
    the best-found fitness values within the cutoff time $\kappa$  are used to
    compare configurations 
    against the actual optimisation time
   for two standard benchmark problem classes, {\scshape Ridge}
    and {\scshape LeadingOnes}.  We rigorously prove that
    all algorithm configurators that use optimisation time as performance metric require cutoff times that  are at
    least as large as the expected optimisation time to identify the optimal configuration. 
    Matters are considerably different if the fitness metric is used.
    To show this we prove that the simple ParamRLS-F configurator can
    identify the optimal mutation rates even when using cutoff times that are considerably
    smaller than the expected optimisation time of the best parameter value for both problem classes. 
\end{abstract}

\maketitle

\section{Introduction}

General purpose search heuristics, such as evolutionary algorithms, are designed with the aim
of optimising a problem given minimal knowledge about it. Usually all that is needed is a means of representing solutions for the problem and a fitness function
to compare the quality of different candidate solutions. 
Whilst these algorithms have been shown to be
effective for solving a large variety of hard optimisation problems,
a common difficulty is that of choosing a suitable algorithm for the problem at hand and setting its parameter values such that it will have
good performance. 
A result of this is that it has become very common to use automated methodologies for algorithm development~\cite{BurkeEtAl2013,paper:integrating_heuristics_constraint_satisfaction_probs,LissovoiEtAl2019,SATenstein}.

Traditionally, parameter values were chosen manually 
by the user, applying the algorithm to a specific problem and subsequently refining the choices according to the algorithm's performance with the tested parameters. This
method, however, is time-consuming, tedious, and error-prone. 
As a result automatic algorithm configurators have gradually become the standard methodology used to tune the parameters of an algorithm for a class of problems.
Popular tuners include ParamILS, which uses
iterated local search to navigate the space of configurations
\cite{paper:paramILS}; irace, which iteratively evaluates many configurations
concurrently, eliminates those which statistically have worst performance,
and then uses the best 
to update the distribution used to generate new candidate configurations~\cite{paper:irace}; 
and the surrogate model-based tuners SPOT~\cite{paper:SPOT}
and SMAC~\cite{paper:ROAR_and_SMAC}, which use approximations of the parameter
landscape in order to avoid many lengthy evaluations of configurations.  \par

Despite their widespread adoption,
there is a lack of understanding of the behaviour and performance of algorithm configurators.
More specifically, given an algorithm
and a class of problems, it is unclear how good the parameter values returned by a given
configurator actually are, 
how long the configurator should be run such that good parameter values are identified,
nor is there any rigorous guidance available on how to set the configurator's inherent parameters.

Recently, Kleinberg \emph{et al.}~\cite{paper:efficiency_through_procrastination} provided preliminary answers to these questions.
They performed a worst-case runtime analysis of standard
algorithm configurators, in which an adversary causes every deterministic choice to play out as poorly as possible, while observations of random variables are unbiased samples from the
underlying distribution. They proved that in this scenario all popular configurators will perform poorly.
On the other hand, they 
presented a worst-case tailored algorithm called Structured Procrastination (SP) that provably performs better in the worst-case.
Several improvements to this approach have recently been published \cite{paper:leaps_and_bounds, paper:caps_and_runs, paper:spc}.
Naturally, it is unlikely that the worst-case scenario occurs in practical applications of algorithm configurators. 
In fact, Pushak and Hoos~\cite{paper:algo_config_landscapes_benign} recently  investigated the structure of  configuration search landscapes. Their experimental analysis suggests that the search landscapes are largely unimodal and convex when tuning algorithms
for well-known instance sets of a variety of NP-hard problems including SAT, MIP and TSP.
Thus, they provided evidence that generally algorithm configuration landscapes are much more benign for popular gradient-based configurators than in the worst-case scenario.
Thus, it is important to rigorously evaluate for which applications a given algorithm configurator will be efficient and for which it will perform poorly.

The 
only available time complexity analysis deriving the time required by an algorithm configurator to identify the optimal parameters for an algorithm
for specific problem classes is an earlier publication of ours \cite{paper:impact_of_cutoff_time_rlsk}. 
We proved that a simplified version of ParamILS, called ParamRLS, can efficiently identify the optimal
neighbourhood size $k$ of a randomised local search algorithm RLS$_k$ for two standard benchmark problem classes.
An important insight gained from our earlier analysis is that, if the best identified fitness within some cutoff time is used for configuration comparisons, then much smaller cutoff times than the actual optimisation time of the optimal configurations may be used by ParamRLS to identify the optimal parameter.
On the other hand, if the optimisation time is used for the comparisons, then 
the cutoff time has to be much larger i.e., at least the expected runtime of the optimal parameter setting.


In order to gain a deeper understanding of the performance of algorithm configurators, in this paper 
we consider the problem
of 
tuning the mutation rate of
the global mutation operator that is commonly used in standard evolutionary algorithms.
The operator,  called {\it standard bit mutation} (SBM),  flips each bit of a bit string of length $n$ with probability $\chi/n$, and $\chi$ is the parameter value to be tuned.
This operator is considerably more sophisticated than the local mutations used by the RLS algorithm considered in our earlier work,
since an arbitrary number of bits between 0 and $n$ may be flipped by the operator in each mutation operation.
Furthermore, the nature of the parameter to be tuned yields a significantly more complex
parameter landscape. While the parameter of RLS$_k$ may only take discrete values, the search space for standard bit mutation's parameter
is continuous as the parameter $\chi$ may take any real value.
Small differences in $\chi$ (e.\,g.\ $1/n$ vs.\ $1.1/n$) are hardly visible as in most mutations the number of flipped bits is identical. In stark contrast, RLS$_k$ behaves very differently when always flipping, say, $k=1$ bit or always flipping $k=2$ bits. Hence, identifying the optimal
standard bit
mutation rate 
is much harder than tuning RLS$_k$ as in our earlier work.




We embed the SBM operator into a simple evolutionary algorithm, the \EA,
and 
consider the problem of identifying its optimal parameter value for two benchmark function classes: {\scshape Ridge}
and {\scshape LeadingOnes}.
The first function class is chosen because for each instance the optimal mutation rate for SBM is always $1/n$ independent of the position of the current solution in the search space.
This characteristic is ideal for a first time complexity analysis as it should be easy for the configurator to identify the optimal parameter value and, at the same time, it keeps the analysis simple.
The second function class is more challenging because the best mutation rate decreases as the algorithm approaches the optimum and the configurator has to identify that the best compromise is
a mutation rate of $\approx 1.59/n$ which minimises the overall expected runtime of the \EA~\cite{paper:opt_mut_rates_1p1_LO}.

Our aim is to characterise the impact of the performance metric on the cutoff time required for algorithm configurators to identify the optimal parameter value of the \EA for the considered problem classes.
As in our publication considering RLS$_k$, in this paper we consider two
performance metrics: {\it Optimisation time}, where the winner of a comparison is
the configuration which reaches the optimum in the fewest iterations; and {\it Best fitness} where the winner of a comparison of
two configurations is the configuration which achieves the highest fitness value
within the cutoff time $\kappa$. 

We prove that, with overwhelming probability (\wop)\footnote{We define an event
$A$ as occurring with overwhelming probability if and only if $\prob(A) = 1 -
\exp(-\Omega(n^\alpha))$, for some positive constant $\alpha$. Note that, by the union bound, the intersection of polynomially many such events still has an overwhelming probability.}, any algorithm configurator that uses Optimisation time as performance metric requires a cutoff time that is at least as large as the optimisation time of the optimal parameter value for both {\scshape Ridge} and {\scshape LeadingOnes}. For smaller cutoff times it returns a parameter value chosen uniformly at random from the parameter space \wop unless the configurator is inherently biased towards some areas of the search space. For the simple {\scshape Ridge}
and {\scshape LeadingOnes} problem classes, a random parameter value is returned \wop respectively for cutoff times of $\kappa \le (1-\varepsilon)en^2$ and $\kappa
\le 0.772075n^2$.
Matters change considerably for algorithm configurators that use Best fitness as performance metric.
To prove this it suffices to consider the simple randomized local search ParamRLS-F algorithm analysed in our earlier work which uses Best fitness as performance metric.
ParamRLS-F efficiently returns the optimal parameter value $\chi=1$ of
the \EA for {\scshape Ridge} for any  cutoff time that is at least linear in the problem size. Notice that the configurator is efficient for cutoff times that are a linear factor smaller than the expected optimisation time of the \EA with the optimal $1/n$ mutation rate.
For {\scshape LeadingOnes}, we prove that, \wop,
ParamRLS-F is able to find the optimal parameter setting of $\chi=1.6$ (where
$\chi$ is allowed to take values from the set $\{0.1, 0.2, \ldots, 2.9, 3.0\}$)
for any cutoff time $\kappa \ge 0.721118n^2$. Note that this is $\approx
0.05n^2$ smaller than the expected optimisation time for any configuration of the
\EA for {\scshape LeadingOnes} \cite{paper:opt_mut_rates_1p1_LO}. For over 99\%
of cutoff times in the range between $0.000001n^2$ and $0.720843n^2$, we prove that
ParamRLS-F returns the optimal parameter setting for the cutoff time,
\wop That is, the smaller the cutoff time, the higher the optimal mutation rate. 
Some proofs are contained in the appendix. For both {\scshape Ridge} and {\scshape LeadingOnes} note that, while the formal proof is provided for the ParamRLS-F tuner, the analysis implies that all algorithm configurators capable of hillclimbing are efficient at tuning the \EA for the same cutoff time values if they use Best fitness as performance metric.

\section{Preliminaries}

We first give an overview of the algorithm configuration
problem. This is a formalisation of the problem which all parameter tuners
attempt to address. We then outline the three main subjects of analysis in this
work. We first present the algorithm configurator. 
We then outline the target algorithm (the algorithm which we are
analysing the ability of ParamRLS to tune) before giving an overview of the two
benchmark target function classes which we consider.

\subsection{The Algorithm Configuration Problem}

The Algorithm Configuration Problem (ACP) is that of choosing the parameters
for a \emph{target algorithm} $\mathcal{A}$ to optimise its performance across
a class of problems $\Pi$. 
Let us denote the set of all configurations of
$\mathcal{A}$ as $\Theta$ and algorithm $\mathcal{A}$ with its parameters set
according to some configuration $\theta \in \Theta$ as $\mathcal{A}(\theta)$.
Then the Algorithm Configuration Problem is the task of identifying a
configuration $\theta^*$ such that
\[ \theta^* \in \arg\min_{\theta \in \Theta} cost(\theta) \]
where $cost(\theta)$ is some measure of the cost of running
$\mathcal{A}(\theta)$ on the problem class $\Pi$. \par

We must therefore define the measure $cost(\theta)$, which depends on several factors: 
The size of the \emph{cutoff time}~$\kappa$ (the number of
iterations in a single run of a comparison); The number of \emph{runs} per
comparison~$r$ (the number of times we evaluate two configurations in a single
comparison); Which \emph{metric} is used to evaluate the performance of a
configuration on a problem instance; How to \emph{aggregate} performance
measures over multiple runs; How many \emph{instances} (and which ones) to
include in the training set. \par

In this work, we address the configurator's parameters as follows. 
All results in this work hold for any polynomial number of runs per comparison (i.e. the positive results hold even for just one run and the negative results hold even for a large polynomial number of runs). We consider two performance metrics. The Optimisation time metric quantifies the
performance of a configuration by the time taken to reach the optimum. A  penalisation constant multiplied by the cutoff time $\kappa$ is returned
if the configuration does not reach the optimum within $\kappa$ iterations (called PAR10 for
penalisation constant of 10). The Best fitness performance metric considers the highest fitness value achieved within the
cutoff time. The definition of the training set and
the method of aggregation are both irrelevant within this paper, since performance on the single
problem instances we consider here generalises to any other instance of the
problem class. \par

Let $T$ be the number of configuration comparisons carried out before the tuner returns the
optimal configuration $\theta^*$ \wop
Then
the total required tuning budget is $\mathcal{B} = 2 \cdot T \cdot | \Pi | \cdot \kappa
\cdot r$. 
In this work, we want to estimate how the two performance metrics impact the cutoff time $\kappa$
and the total number of comparisons~$T$ required for a simple
algorithm configurator ParamRLS to tune the \EA for two benchmark problem
classes. 

\subsection{The Configurator: ParamRLS}

We follow our earlier work in analysing a simplified version
of the popular ParamILS parameter tuner, called ParamRLS~\cite{paper:impact_of_cutoff_time_rlsk}.

At the heart of ParamRLS is the \emph{active parameter}. ParamRLS initialises
the active parameter uniformly at random. It then repeatedly mutates it and
accepts the offspring (that is, updates the active parameter to this new value)
if it performs at least as well as the active parameter according to a routine
\texttt{eval}. Each call to \texttt{eval} is called a \emph{comparison} and
the \emph{runtime} $T$ is defined as the number of comparisons until the optimal
parameter value is identified. The high-level pseudocode for ParamRLS is given
in Algorithm~\ref{algo:pRLS}. The \texttt{eval} routine takes as arguments both
configurations to be compared, as well as the cutoff time $\kappa$ and number
of runs $r$. Both configurations are then executed for $r$ runs (each of length
$\kappa$) where at the end of each run the winner of the run is decided
according to one of two \emph{performance metrics}.
Under 
the Best fitness metric, the winner of a run is the configuration which has the highest fitness
value after $\kappa$ iterations. If both configurations have the same fitness
value at time $\kappa$, then the winner is the one that found it first.
The winner of the overall comparison
is the configuration which won the most runs. Ties are broken uniformly at
random. The variant of ParamRLS using this performance metric is called
ParamRLS-F, and its pseudocode is given in Algorithm~\ref{algo:evalF}. Under the Optimisation time metric,
the optimisation times of both comparisons are
summed for $r$ runs.  If in a run a configuration fails to reach the optimum
within $\kappa$ iterations then its optimisation time is taken to be $p \cdot
\kappa$, where $p$ is a penalty constant. If there is a tie after all runs have
been completed, the winner is decided uniformly at random. This measure is
called Penalised Average Runtime (PAR). It is commonly used in configurators
such as ParamILS. The variant of ParamRLS using this performance metric is
called ParamRLS-T, and its pseudocode is given in Algorithm~\ref{algo:evalT}.
\par

As in \cite{paper:impact_of_cutoff_time_rlsk}, the \texttt{mutate} routine in
ParamRLS alters the current configuration according to some \emph{local search
operator} $\pm\{1\}$. This operator simply increases or decreases the current
parameter value by~$1/d$, both with probability 0.5. If the new value for the
parameter is 0 or $\phi+(1/d)$ then this new configuration loses any comparison
with probability 1 (that is, it is rejected with certainty). \par

Since we consider a continuously-valued parameter, $\chi$, we discretise the parameter space (the search space of all
configurations) as this is also the method used to deal with
continuous parameters in ParamILS \cite{paper:paramILS}. We do so using a
\emph{discretisation factor}~$d$. We define the parameter
space as consisting of all numbers $z/d$ where $z \in \{1,2,3,\ldots,\phi \cdot
(d-2),\phi \cdot (d-1),\phi \cdot d\}$, for an integer~$\phi$. For example,
if $d=4$ then the possible values for $\chi$ will be $0.25, 0.5, 0.75, 1, 1.25,\allowbreak 1.5,\allowbreak 1.75,\allowbreak 2, \ldots, \phi-0.25, \phi$. We do not consider
$\chi=0$ since this means that bits are never flipped in the
target algorithm. We call the fitness landscape induced by the target algorithm on the target function
under a performance metric the \emph{configuration landscape}. We say that a configuration landscape is \emph{unimodal} if and only if there is only one optimal
configuration and, for all pairs of neighbouring configurations (two
configurations are neighbours if one can be reached from the other in a single
mutation using the local search operator), the configuration closer to the
optimum wins a comparison \wop  \par

We call a tuner \emph{blind} if there is an event~$A$ that occurs \wop and, conditional on~$A$, the tuner returns a configuration chosen according to a distribution which would be generated if all configurations had the same performance. For ParamRLS-T this implies that the configuration will be chosen uniformly at random, since if there is no information to separate two configurations then the winner is chosen uniformly at random. Only if a tuner is inherently biased will the outcome be non-uniformly distributed.

\begin{algorithm}[t]

    \begin{algorithmic}[1]
        \STATE{$\theta \gets $initial parameter value chosen uniformly at random}
        \WHILE{termination condition not satisfied}
            \STATE{$\theta' \gets \text{\texttt{mutate($\theta$)}}$} 
                \STATE{$\theta \gets \text{\texttt{eval}($\mathcal{A},\theta,\theta',\pi,\kappa,r$)}$}
        \ENDWHILE
        \RETURN $\theta$
    \end{algorithmic}

    \caption{ParamRLS ($\mathcal{A},\Theta,\Pi,\kappa,r$). Recreated from \cite{paper:impact_of_cutoff_time_rlsk} with minor typographical modifications.}
    \label{algo:pRLS}

\end{algorithm}

\begin{algorithm}[t]

    \begin{algorithmic}[1]
        \STATE $Wins \gets 0$; $Wins' \gets 0$ \COMMENT{count number of wins for $\theta$ and $\theta'$}
        \STATE $R \gets 0$
        \WHILE {$R < r$}
            \STATE $ImprTime \gets 0$
            \STATE $ImprTime' \gets 0$
            \STATE $Fitness \gets \mathcal{A}(\theta)$ fitness after $\kappa$ iterations; 
            \STATE $Fitness' \gets \mathcal{A}(\theta')$ fitness after $\kappa$ iterations; 
            \STATE $ImprTime \gets $time of last improvement of $\mathcal{A}(\theta)$
            \STATE $ImprTime' \gets $time of last improvement of $\mathcal{A}(\theta')$
            \IF{$Fitness > Fitness'$}
                \STATE $Wins \gets Wins + 1$
            \ELSIF{$Fitness' > Fitness$}
                \STATE $Wins' \gets Wins' + 1$
            \ELSE
                \IF{$ImprTime < ImprTime'$}
                    \STATE $Wins \gets Wins + 1$
                \ELSIF{$ImprTime' < ImprTime$}
                    \STATE $Wins' \gets Wins' + 1$
                \ENDIF
            \ENDIF
            \STATE  {$R \gets R+1$}
        \ENDWHILE
        \STATE \algorithmicif\ $Wins > Wins'$ \algorithmicthen\ \algorithmicreturn\ $\theta$
        \STATE \algorithmicelsif\ $Wins' > Wins$ \algorithmicthen\ \algorithmicreturn\ $\theta'$
        \STATE \algorithmicelse\ \algorithmicreturn\ a uniform choice of $\theta$ or $\theta'$
    \end{algorithmic}

    \caption{The \texttt{eval-F}($\mathcal{A},\theta,\theta',\pi,\kappa,r$) subroutine in ParamRLS-F. Recreated from \cite{paper:impact_of_cutoff_time_rlsk} with minor typographical modifications.}
    \label{algo:evalF}

\end{algorithm}

\begin{algorithm}[t]

    \begin{algorithmic}[1]
        \STATE $Time \gets 0$; $Time' \gets 0$ \COMMENT{count optimisation times for $\mathcal{A}(\theta)$ and $\mathcal{A}(\theta')$}
        \STATE $R \gets 0$
         \WHILE {$R < r$}
            \STATE $Time \gets Time + \text{\texttt{CapOptTime}}(\mathcal{A}(\theta), \kappa, p)$\!\!\!\!\!\!\!\!\!\!\!\!\!\!\!\!\!\!\!\!
            \STATE $Time' \gets Time' + \text{\texttt{CapOptTime}}(\mathcal{A}(\theta'), \kappa, p)$\!\!\!\!\!\!\!\!\!\!\!\!\!\!\!\!\!\!\!\!
            \STATE  {$R \gets R+1$}
        \ENDWHILE
        \STATE \algorithmicif\ $Time < Time'$\ \algorithmicthen\ \algorithmicreturn\ $\theta$
        \STATE \algorithmicelsif\ $Time' < Time$\ \algorithmicthen\ \algorithmicreturn\ $\theta'$
        \STATE \algorithmicelse\ \algorithmicreturn\ a uniform choice of $\theta$ or $\theta'$
    \end{algorithmic}

    \caption{The \texttt{eval-T}($\mathcal{A},\theta,\theta',\pi, \kappa,r$) subroutine in ParamRLS-T. Recreated from \cite{paper:impact_of_cutoff_time_rlsk} with minor typographical modifications.}
    \label{algo:evalT}

\end{algorithm}

\subsection{The Target Algorithm: The \EA}

In each iteration, the \EAc, shown in Algorithm~\ref{algo:(1+1)_EA}, creates a new solution by flipping each bit in the bit string of the current solution 
independently with probability $\chi/n$. The offspring is accepted if its
fitness is at least that of its parent. 
We analyse the number of  comparisons for ParamRLS to identify the optimal value for $\chi$ (and
thus the optimal mutation rate) for a given problem class. In this work we
assume that $\chi$ is constant.

\begin{algorithm}
    \begin{algorithmic}[1]

        \STATE{\textbf{initialise $x$}} \COMMENT{according to initialisation scheme} \label{line:init}
        \WHILE{termination criterion not met}
        \STATE{$x' \gets$ $x$ with each bit flipped with probability $\chi/n$}
        \STATE{\algorithmicif\ $f(x') \ge f(x)$ \algorithmicthen\ $x \gets x'$} \label{line:fxprime_ge_fx}
        \ENDWHILE

    \end{algorithmic}

    \caption{The \EA maximising a function $f$}
    \label{algo:(1+1)_EA}

\end{algorithm}

\subsection{Problem Classes: {\scshape Ridge} and {\scshape LeadingOnes}}

We analyse the ability of ParamRLS to configure the \EA for two standard
benchmark problem classes: {\scshape Ridge} and {\scshape LeadingOnes}. Despite
their similar definitions, the \EA exhibits substantially different behaviour
when optimising them. 

The most commonly
analysed instance of the {\scshape Ridge} problem
class 
is the function
\[
    \text{{\scshape Ridge}}(x) =
        \begin{cases}
            n + |x|_{\text{\scshape ones}}, \text{ if } x \text{ in form } 1^{i}0^{n-i} \\
            n - |x|_{\text{\scshape ones}}, \text{ otherwise}
        \end{cases}
\]
where $|x|_{\text{\scshape ones}}$ is the number of ones in the bit string $x$. \par
%
The other instances of the problem class are provided by its black box definition, 
given by \cite{DrosteJansenWegener2002b} and used by \cite{paper:impact_of_cutoff_time_rlsk}, which in a nutshell takes the XOR with another bit string $a \in \{0, 1\}^n$: $\text{{\scshape Ridge}}_a :=
\text{{\scshape Ridge}}{}(x_1 \oplus a_1 \dots x_n \oplus a_n)$. We thus
analyse only the problem instance $\text{{\scshape Ridge}}_{0^n}$, and observe
that the best parameter value for this problem instance will be optimal for all
$2^n$ instances of the problem class. \par

Following previous work~\cite{paper:impact_of_cutoff_time_rlsk,Jansen2014}, we assume that the \EA is initialised to~$0^n$. This means that the algorithm builds a string of
consecutive 1 bits followed by a string of consecutive 0 bits, and the optimum is the $1^n$ bit string.

For the \EA, it is optimal to set $\chi=1$ to achieve the smallest
expected optimisation time for {\scshape Ridge}. We prove this in
Lemma~\ref{lem:ridge_facts}(\ref{lem:proof_c_1_optimal_ridge}).
The {\scshape Ridge} problem class is a natural one to analyse initially since
it is always best to use a mutation rate of $1/n$. 
This characteristic simplifies the analysis. \par

The second problem class is {\scshape LeadingOnes}, or `LO' for short. In the problem instance we consider, the
fitness value of a bit string is equal to the number of consecutive 1 bits at
the beginning of the string, $\text{\scshape LO}(x) = \sum_{i=1}^{n} \prod_{j=1}^{i} x_{j}$. This differs from {\scshape Ridge} since
{\scshape Ridge} requires the bit string always be in the form $1^i0^{n-i}$,
whereas LO places no such importance on the bits following
the first~0.

Droste \textit{et al.} define the black box optimisation class of {\scshape
LO} as the class consisting of problem instances {\scshape
LO}$_a(x)$, where this is taken to be the length of the longest prefix
of $x$ in which all bits match the prefix of $\overline{a}$~\cite{DrosteJansenWegener2002b}.
Naturally, the best mutation rate for one instance will also be optimal for all the other instances in the problem class.

B{\"o}ttcher \textit{et al.} proved that setting $\chi = 1.59\ldots$ leads to the shortest
expected optimisation time for the \EA for the LO problem
class for any static mutation rate \cite{paper:opt_mut_rates_1p1_LO}.
LO presents a
more complex problem for which to tune than that presented by {\scshape Ridge},
since it is beneficial to use a higher mutation rate earlier in the
optimisation process. Intuitively, this is because it is necessary to preserve
the current prefix of leading ones in order to make progress, thus as
the prefix grows, higher mutation rates are more likely to flip bits within it.
Hence it is challenging to determine the behaviour of the tuner for
different cutoff times. 

\section{Tuning the \EA for {\scshape Ridge}}
\label{sec:ridge}

Before we can prove any results on the performance of ParamRLS when
tuning the \EA for {\scshape Ridge}, it is first necessary to analyse the
performance of the \EA on this function. We therefore begin this section by
deriving a lemma which tells us several facts about its behaviour. We first bound the \emph{drift} (the change of the fitness of the
current individual) in one generation of the \EA.  We denote the drift of the \EAc on {\scshape Ridge} by
$\Delta_{\chi}$, and define it as $\Delta_{\chi}(x_{t}) := \text{\scshape Ridge}(x_{t+1})-\text{\scshape Ridge}(x_t)$.
The following lemma summarises key statements about the performance of the \EAc on {\scshape Ridge}.

\begin{lemma}
    \label{lem:ridge_facts}
    Consider the \EA optimising {\scshape Ridge}. Assume that it was
    initialised to $0^n$. Then the following is true:
    \begin{enumerate}[(i)]
        \item{ \label{lem:exp_drift_1p1_a_ridge}
            The expected drift of the \EAc, $\E[\Delta_{\chi}(x_{t})
            \mid x_{t}]$, is bounded as follows:
            \begin{align*}
             \frac{\chi}{n} \left( 1 - \frac{\chi}{n} \right)^{n-1} \le\;& \E[\Delta_{\chi}(x_{t}) \mid x_{t}, x_t \neq 1^n]\\
             \E[\Delta_{\chi}(x_{t}) \mid x_{t}] \le\;& \frac{\chi}{n} \left( 1 - \frac{\chi}{n} \right)^{n-1} + O \left( \frac{1}{n^{2}} \right)
             \end{align*}}
        \item{\label{lem:proof_c_1_optimal_ridge}
            Setting $\chi=1$ yields the smallest expected optimisation time for
            any constant $\chi$, which is at most $en^2$, assuming $n$ is large enough.}
    \end{enumerate}
\end{lemma}
The first statement follows as the probability of improving the fitness of $x_t\neq 1^n$ by~1 is $\chi/n \cdot (1-\chi/n)^{n-1}$ as it is necessary and sufficient to flip the first 0-bit and not to flip the other $n-1$ bits.
Larger jumps have an exponentially decaying probability, reflected in the $O(1/n^2)$ term.
For the second statement we use additive drift theory~\cite{He2001,DBLP:journals/algorithmica/Kotzing16}, which in a nutshell derives (bounds on) first hitting times from (bounds on) the expected drift and the initial distance from the target state. Hence the expected optimisation time is minimised by the parameter that maximises the drift. The second statement follows from
the first one since\footnote{We use an $\approx$ symbol for illustration purposes only. Proofs use the double inequality $(1-\chi/n)^{n} \le e^{-\chi} \le (1-\chi/n)^{n-\chi}$.} $ \left( 1 - \chi/n \right)^{n-1} \approx e^{-\chi}$ and the function $\chi e^{-\chi}$ is maximised for $\chi=1$.


\subsection{Analysis of ParamRLS-F Tuning for {\scshape Ridge}}

We now prove that, for any discretisation factor $d$, ParamRLS-F
is able to configure the \EA for {\scshape Ridge}. In particular, we show that,
given any cutoff time $\kappa \ge c'n$, for a sufficiently large
constant ${c'>0}$, the expected number of comparisons in ParamRLS-F before the
active parameter is set to $\chi=1$ is at most~$6d\phi$. Moreover,  after
$dn^{\varepsilon}$ comparisons with cutoff time $\kappa \ge
n^{1+\varepsilon}$, for any positive constant~$\varepsilon$,
ParamRLS-F returns the optimal configuration \wop

\begin{theorem}
    \label{thm:prlsf_can_tune_for_ridge}
    Consider ParamRLS-F tuning the \EA for {\scshape Ridge}, where the target
    algorithm is initialised to $0^n$. Assume that $d, \phi=\Theta(1)$ and $\kappa \in
    \mathrm{poly}(n)$. Then:
\begin{itemize}
        \item{Using cutoff times $\kappa \ge c'n$
            for a sufficiently large constant ${c' > 0}$, the expected number of
            comparisons in ParamRLS-F before the active parameter value has
            been set to $\chi=1$ is at most $6d\phi$.}
        \item{Using cutoff times $\kappa \ge n^{1+\varepsilon}$, for some
            constant $\varepsilon>0$, if ParamRLS-F runs for
            $dn^{\varepsilon}$ comparisons
            then it returns the parameter value $\chi=1$ with overwhelming
            probability.}
    \end{itemize}
\end{theorem}

We prove the above theorem by bounding the probability that one configuration
has a higher fitness than another after $\kappa$ iterations.
For large enough cutoff times, in a
comparison between two configurations which either both have $\chi\le1$ or both
have $\chi\ge1$, the configuration with $\chi$ closer to 1 wins.
\begin{lemma}
    \label{lem:prob_1p1_a_beats_1p1_b}
    Assume that the \EAc[a] and the \EAc[b], with $a$ and $b$ two
    positive
    constants such that $ae^{-a} > be^{-b}$, are both initialised to $0^n$.
    Then with probability at least
    \[ 1-3\exp(-\Omega(\kappa/n))-\kappa \exp(-\Omega(n)) \]
    the \EAc[a] wins in a comparison in ParamRLS-F against the \EAc[b] on
    {\scshape Ridge} with cutoff time $\kappa$. Note that if $a$ and $b$
    satisfy either $0 < b < a \le 1$ or $1 \le a < b \le \phi$ then the
    condition $ae^{-a} > be^{-b}$ is implied.
\end{lemma}
The lemma follows from showing, through the use of appropriate Chernoff bounds, that either the \EAc[a] is ahead of the \EAc[b] after $\kappa$ iterations or that the \EAc[a] finds the optimum sooner than the \EAc[b].

%

Now we are able to prove Theorem~\ref{thm:prlsf_can_tune_for_ridge}.
\begin{proof}[Proof of Theorem~\ref{thm:prlsf_can_tune_for_ridge}]
    Given a comparison of a pair of
    configurations, let us call the configuration with a value of $\chi$ closer
    to 1 the `better' configuration, and the other configuration the
    `worse' configuration.

    By Lemma~\ref{lem:prob_1p1_a_beats_1p1_b}, using $\kappa \ge c'n$, the probability that the better configuration wins a comparison with cutoff time~$\kappa$ is at least $1-\exp(-\Omega(\kappa/n))-\kappa \exp(-\Omega(n)) \ge 2/3$, the inequality holding since we can choose the constant~$c' > 0$ appropriately and $\kappa \in \mathrm{poly}(n)$.
    The current configuration is compared against a better one with probability
    at least $1/2$, and it is compared against a worse one with probability
    at most $1/2$.
    Hence the distance to the optimal parameter value decreases in
    expectation by at least $1/d \cdot (1/2 \cdot 2/3 - 1/2 \cdot 1/3) =
    1/(6d)$. The initial distance is at most $\phi$.
    By additive drift arguments (Theorem~5 in~\cite{DBLP:journals/algorithmica/Kotzing16}), the expected time to reach the optimal
    parameter value for the first time is at most $6d\phi$.

    For the second statement, we use that if $\kappa \ge n^{1+\varepsilon}$
    then the probability of accepting a worse configuration is exponentially
    small. Hence, \wop, within any polynomial number of comparisons, we never experience the event that the worse configuration
    wins a comparison.
    This implies that $\max(1, \phi-1)(d+1)$ steps decreasing the distance
    towards the optimal parameter are sufficient.
    By Chernoff bounds, the probability of not seeing this many steps in
    $dn^{\varepsilon}$ iterations is exponentially small. Finally, once the
    optimal parameter is reached, it is never left \wop Thus, after
    $dn^{\varepsilon}$ iterations, the optimal parameter is returned with
    overwhelming probability.
\end{proof}


Lemma~\ref{lem:prob_1p1_a_beats_1p1_b} implies that any
parameter tuner capable of hillclimbing will return the optimal configuration \wop,
given sufficiently many comparisons, if it uses cutoff times of $\kappa \ge
n^{1+\varepsilon}$ and the highest fitness
value performance metric.

\subsection{Analysis of Optimisation Time-Based Comparisons When Tuning for {\scshape Ridge}}
\label{sec:opt_time_based_blind_ridge}

While ParamRLS-F succeeds at
tuning the \EA for {\scshape Ridge} with a
cutoff time of $\kappa \ge n^{1+\varepsilon}$, we now show that
all algorithm configurators that use Optimisation time as performance metric fail \wop to identify the optimal configuration if the cutoff time is at most $\kappa \le (1-\varepsilon)en^2$.
Established tuners such as irace, ParamILS, and SMAC as well as recent theory-driven approaches such as Structured Procrastination all fall into this category if Optimisation time is used as performance metric.
For such cutoff times, all configurations (i.e. the target algorithm with any parameter value) fail to find the optimum \wop

\begin{lemma}
    \label{lem:ridge_opt_time_lower_bound}
    For all constants $\chi, \varepsilon > 0$, the \EAc[\chi] requires more
    than $(1-\varepsilon)en^2$ iterations to reach the optimum of {\scshape
    Ridge}, with probability $1-\exp(-\Omega(n))$.
\end{lemma}

This yields that all configurators that use the Optimisation time performance metric are blind if the cutoff time is at most
$\kappa \le (1-\varepsilon)en^2$.

\begin{theorem}
    \label{thm:prlst_fails_small_cutoffs_rdg}
    Consider any configurator using the Optimisation time performance metric tuning the \EA for {\scshape Ridge} for any positive
    constant $\phi$ and discretisation factor~$d$. If the cutoff time for each run is never allowed to exceed $\kappa \le (1-\varepsilon)en^2$, for some constant
    $\varepsilon>0$, then, after any polynomial number of comparisons and runs per
    comparison, the configurator is blind.
\end{theorem}
\begin{proof}[Proof of Theorem~\ref{thm:prlst_fails_small_cutoffs_rdg}]
    Lemma~\ref{lem:ridge_opt_time_lower_bound} tells us
    that, for cutoff times $\kappa \le (1-\varepsilon)en^2$, all configurations of the \EAc[\chi], for every constant choice of~$\chi$, fail to reach the optimum of {\scshape Ridge}, with
    overwhelming probability. If this happens in all comparisons, then, since the Optimisation time metric is being used, all configurations will have the same fitness: $\kappa$
    multiplied by the penalisation constant. Thus, there is no
    attraction towards the optimal parameter value. Therefore, with overwhelming probability, by the union bound, the configurator will behave as if all configurations have the same performance, and therefore it is blind.
\end{proof}


\section{Tuning the \EA for {\scshape LeadingOnes}}
\label{sec:lo}


We now show that ParamRLS-F is able to find optimal parameter values for the \EAc[\chi] optimising {\scshape LeadingOnes} for almost all quadratic cutoff times~$\kappa$. The analysis is considerably more complicated than for {\scshape Ridge} since the progress depends (mildly, but significantly) on the current fitness. For a search point with $k$ leading ones, the probability of improving the fitness is exactly $\chi/n \cdot (1-\chi/n)^{k}$ as it is necessary and sufficient to flip the first 0-bit while not flipping the $k$ leading ones. This probability decreases over the course of a run, from $\chi/n \cdot (1-\chi/n)^0 = \chi/n$ for $k=0$ to $\chi/n \cdot (1-\chi/n)^{n-1} \approx \chi/(en)$ for $k=n-1$. This effect is similar to that observed for the function {\scshape OneMax} in~\cite{paper:impact_of_cutoff_time_rlsk}. We therefore follow our approach in that work and establish intervals that bound the ``typical'' fitness at various stages of a run. This allows us to locate the final fitness after $\kappa$ iterations with high precision and \wop For almost all cutoff times, our fitness intervals reveal that the configuration closer to the optimal one leads to a better final fitness, \wop

Due to the increased complexity of the analysis, we focus on one specific discretisation factor $d$ and choice of $\phi$ as a proof of concept.
We are confident that our method generalises to any constant discretisation factor by increasing the precision of our analytical results (by means of the period length introduced in Lemma~\ref{lem:lo_progress_bounds})
in relation to the granularity
of the parameter space (given by~$d$).  
We provide a Python tool which
applies our proof technique for an arbitrary period length\footnote{Available at \url{https://github.com/george-hall-sheff/leading_ones_recurrences_tool}.}. Therefore the user
can repeatedly decrease the period length until this tool is able to prove the
desired results for their chosen parameter space. \par

We choose a discretisation factor of $d=10$ and $\phi=3$, which implies $\chi \in \{0.1, 0.2, \ldots,
2.9, 3.0\}$.
The
mutation rate which produces the smallest expected optimisation time for the
\EA optimising LO is $\approx 1.59/n$~\cite{paper:opt_mut_rates_1p1_LO}, and it is easily verified that the optimal parameter with the chosen granularity is~$\chi=1.6$. We expect the tuner to return $\chi=1.6$ when the cutoff time is large enough.
For smaller cutoff times we expect the tuner to return larger values
of~$\chi$, since for LO it is beneficial to flip more bits
when early in the optimisation process. 
We prove that, for ParamRLS-F,
both of these intuitions are correct. However, tuners using the Optimisation time performance metric require larger cutoff times
in order to identify the optimal configuration.

\subsection{Analysis of ParamRLS-F Tuning for {\scshape LO}}

In this section we prove two results. We first prove that the parameter landscape
is unimodal for over 99\% of cutoff times in the range
$0.000001n^2$ to $0.772075n^2$. We then prove that the parameter landscape is
unimodal for all cutoff times of at least $0.772076n^2$, and that
the optimal mutation rate (and that returned by ParamRLS\nobreakdash-F) for these cutoff
times is $1.6/n$, as expected. These results imply that, given sufficiently many
comparisons, ParamRLS\nobreakdash-F will, \wop, return the
mutation rate with the smallest expected optimisation time for all cutoff times
of at least $0.772076n^2$, and for over 99\% of cutoff times in the range
$0.000001n^2$ to $0.772075n^2$ it will return the mutation rate which is
optimal (that is, it achieves the highest fitness) for that cutoff time. \par

\begin{theorem}
    \label{thm:prlsf_can_tune_for_lo}
    Consider ParamRLS-F tuning the \EA for LO with $\chi \in \{0.1, 0.2, \ldots, 2.9, 3.0\}$ (i.\,e. $d=10, \phi=3$). For all cutoff
    times in one of the ranges listed in Table~\ref{tab:param_space_monotonic_ranges} and $\kappa \ge 0.772076n^2$ it holds that, for any positive constant~$\varepsilon$:
    \begin{itemize}
        \item{The expected number of comparisons in ParamRLS-F before the
            active parameter is set to the optimal value for the
            cutoff time (see Table~\ref{tab:param_space_monotonic_ranges}) is at most $2d\phi + \exp(-\Omega(n^{\varepsilon}))$.}
        \item{If ParamRLS-F is run for a number of comparisons which is both polynomial and at least $n^{\varepsilon}$ then
            it returns the optimal parameter value for the cutoff time with
            overwhelming probability.}
    \end{itemize}
\end{theorem}

\begin{table}[t]
    \centering
    \begin{tabular}{| c | c c |}
        \hline
        $\chi$ & lower bound on $\kappa$ & upper bound on $\kappa$ \\
        \hline
            3.0 & $0.000030n^2$ & $0.225138n^2$ \\
            2.9 & $0.225628n^2$ & $0.241246n^2$ \\
            2.8 & $0.241720n^2$ & $0.259143n^2$ \\
            2.7 & $0.259600n^2$ & $0.279105n^2$ \\
            2.6 & $0.279545n^2$ & $0.301461n^2$ \\
            2.5 & $0.301885n^2$ & $0.326611n^2$ \\
            2.4 & $0.327018n^2$ & $0.355040n^2$ \\
            2.3 & $0.355431n^2$ & $0.387346n^2$ \\
            2.2 & $0.387720n^2$ & $0.424266n^2$ \\
            2.1 & $0.424623n^2$ & $0.466723n^2$ \\
            2.0 & $0.467064n^2$ & $0.515884n^2$ \\
            1.9 & $0.516208n^2$ & $0.573238n^2$ \\
            1.8 & $0.573546n^2$ & $0.640714n^2$ \\
            1.7 & $0.641006n^2$ & $0.720843n^2$ \\
            1.6 & $0.721118n^2$ & $0.772075n^2$ \\
        \hline
    \end{tabular}
    \caption{Ranges of $\kappa$ for which the parameter landscape is unimodal
    with the optimum at $\chi$. The parameter landscape is also
    unimodal with the optimum at $\chi=1.6$ for cutoff times $\kappa \ge
    0.772076n^2$, as shown in Theorem~\ref{thm:prlsf_can_tune_for_lo}.}
    \label{tab:param_space_monotonic_ranges}
\end{table}

In order to prove Theorem~\ref{thm:prlsf_can_tune_for_lo}, we first bound the
progress made by the \EA in $n^2/\psi$ iterations (for a positive constant
$\psi$), a length of time we call a \emph{period}. We define progress as the difference between the distance to the optimum at the beginning of the period and at the end of the period. We then sum the progress
made in a constant number of periods in order to bound the fitness of the
individual in the \EA after a quadratic number of iterations. We compute the
cutoff times required such that these intervals do not overlap. This tells us
which configuration will win in a comparison of that length. \par

We first derive progress bounds for the \EAc[\chi].
\begin{lemma}
    \label{lem:lo_progress_bounds}
    Consider the \EA optimising LO for a period of
    $n^2/\psi$ iterations, for some positive constant $\psi$, that starts with a fitness of~$j$. Let $Z$ be
    the amount of progress made by the algorithm
    over the period. Then, \wop:
    \begin{enumerate}[(i)]
        \item{\label{lem:lo_progress_ub} $Z \le \frac{2 \chi n}{\psi \cdot \exp \left( \frac{\chi j}{n} \right)} + o(n)$}
        \item{\label{lem:lo_progress_lb} For every~$i$ with $j \le i < n$, $Z \ge
            \frac{2 \chi n}{\psi \cdot \exp \left( \frac{\chi i}{n} \right)} - o(n)$, or
            the algorithm exceeds fitness $i$ at the end of the period.}
    \end{enumerate}
\end{lemma}
The intuition behind these bounds is that the probability of improving the fitness of a search point with $k$ leading ones is $\chi/n \cdot (1-\chi/n)^{k}$, which is at least $\chi/n \cdot (1-\chi/n)^i \approx \chi/n \cdot \exp(-\chi i/n)$ and at most $\chi/n \cdot (1-\chi/n)^j \approx \chi/n \cdot \exp(-\chi j/n)$ if $j \le k \le i$. The factor of~2 stems from the fact that, when the first 0-bit is flipped, the fitness increases by $2$ in expectation as the following bits may be set to~1.

By applying the progress bounds from Lemma~\ref{lem:lo_progress_bounds} inductively, we derive the following bounds on the current fitness of the \EA after an arbitrary number of periods.
\begin{lemma}
    \label{lem:lo_fitness_bounds}
    Consider the \EA optimising LO. Let a run of length
    $\alpha n^2$ be split into $\alpha \psi$ periods of length $n^2/\psi$ (for
    positive constants $\alpha$ and $\psi$). Define $\ell_{\chi,0} := 0$ and
    $u_{\chi,0} := \sqrt{n}$. Then for $i \le \alpha$ there exist $u_{\chi, i+1}$ and $\ell_{\chi, i+1}$ with
    \begin{align*}
     u_{\chi,i+1} =\;& u_{\chi,i} + \frac{2 \chi n}{\psi \exp \left( \frac{\chi u_{\chi,i}}{n} \right)} + o(n)\\
     \ell_{\chi,i+1} =\;& \ell_{\chi,i} + \frac{2 \chi n}{\psi \exp \left( \frac{\chi u_{\chi,i+1}}{n} \right)} - o(n)
    \end{align*}
    such that, with overwhelming probability, the following holds. At the end of period $i$, for $0 \le i
    \le \alpha$, the current fitness is in the interval
    $[\ell_{\chi,i},u_{\chi,i}]$ or an optimum has been found, and throughout
    period $i$ the fitness is in $[\ell_{\chi,i-1},u_{\chi,i}]$ or
    an optimum has been found.
\end{lemma}

Since we do not have a closed form for the intervals derived in Lemma~\ref{lem:lo_fitness_bounds}, we follow the approach in~\cite{paper:impact_of_cutoff_time_rlsk} and iterate them computationally in
order to derive exact bounds on the fitness after a given number of iterations, using our Python tool mentioned earlier.
We observe that the \EA makes, in expectation, a linear amount of progress
during a period which consists of a quadratic number of iterations (as is the
case in Lemma~\ref{lem:lo_fitness_bounds}). This fact implies that we can check
whether one configuration is ahead of another by computing the leading constant of
the $\Theta(n)$ term in the fitness bounds from
Lemma~\ref{lem:lo_fitness_bounds} and check whether the intervals are
overlapping. If they are not overlapping then, \wop, one configuration is ahead of another by a linear
amount. If $n$ is large enough then the $o(n)$ terms from Lemma~\ref{lem:lo_fitness_bounds} can be ignored as the fitness is determined exclusively be the leading constants of the linear terms. We extract the relevant leading constants from the fitness bounds in the
following lemma.

\begin{lemma}
    \label{lem:recurrence_defs}
    Let $c_{\ell,\chi,i}$ and $c_{u,\chi,i}$ denote the leading constants in the
    definition of $\ell_{\chi,i}$ and $u_{\chi,i}$ from
    Lemma~\ref{lem:lo_fitness_bounds}, respectively (i.\,e., $\ell_{\chi,i} =
    c_{\ell,\chi,i} \cdot n - o(n)$ and $u_{\chi,i} = c_{u,\chi,i} \cdot n + o(n)$).
    Then
    $c_{u,\chi,i+1}$ and $c_{\ell,\chi,i+1}$ can be expressed using the recurrences
    $c_{\ell, \chi, 0} = c_{u, \chi, 0} = 0$,
    \begin{align*}
     c_{u,\chi,i+1} =\;&  c_{u,\chi,i} + \frac{2\chi}{\psi \exp \left( \chi \cdot c_{u,\chi,i} \right)}\\
    c_{\ell,\chi,i+1} =\;& c_{\ell,\chi,i} + \frac{2\chi}{\psi \exp \left( \chi \cdot c_{u,\chi,i+1} \right)}.
    \end{align*}
\end{lemma}

We can now prove that the parameter landscape is unimodal.

\begin{lemma}
    \label{lem:lo_param_space_unimodal_concave}
    For ParamRLS-F tuning the \EA for LO, the parameter
    values in the set $\{1.6,1.7,\ldots,2.9,3.0\}$ are optimal for the ranges
    of $\kappa$ given in Table~\ref{tab:param_space_monotonic_ranges}, if $n$ is large enough. Furthermore, the parameter landscape is unimodal for these cutoff times.
\end{lemma}

Having proved that the parameter landscape is unimodal for the cutoff
times given in Table~\ref{tab:param_space_monotonic_ranges}, we now turn our
attention to cutoff times $\kappa \ge 0.772076n^2$. We prove that, for cutoff
times in this range, the parameter value $\chi=1.6$ wins ParamRLS-F comparisons
against any other configuration \wop, and again the
parameter landscape is unimodal. In order to prove this result, it is
first necessary to prove a helper lemma. This lemma takes two configurations
and the distance between them and gives a condition which, if satisfied,
implies that, \wop, the configuration which is closer
to the optimum reaches the optimum before the configuration which is behind
covers the initial distance between them. That is, the configuration which is
closer to the optimum wins a ParamRLS\nobreakdash-F comparison with overwhelming
probability. \par

\begin{lemma}
    \label{lem:remains_ahead_ineq}
    Assume that the fitness of the individuals in the \EAc[a] and the \EAc[b] optimising LO are contained in the intervals $[c_{\ell,a,i} \cdot n-o(n),c_{u,a,i} \cdot n+o(n)]$ and $[c_{\ell,b,i} \cdot n-o(n),c_{u,b,i} \cdot n+o(n)]$, respectively, as defined in Lemma~\ref{lem:recurrence_defs}. Assume that $c_{\ell,a,i} > c_{u,b,i}$ (that is, the \EAc[a] is ahead of the \EAc[b] by some linear distance). If
    \[ \frac{\left( \frac{2 b}{(c_{\ell,a,i}-c_{u,b,i}) \cdot \exp \left( bc_{\ell,b,i} \right)} + \varepsilon \right)}{\left( \frac{2a}{(1 - c_{\ell,a,i}) \cdot \exp \left( a \right)} \right)} \le 1 \]
    for some positive constant $\varepsilon$ then, with
    overwhelming probability, the \EAc[a] reaches the optimum before the
    \EAc[b] has covered the initial distance between the two algorithms.
\end{lemma}


We now use this lemma to prove that the parameter landscape is unimodal
for cutoff times $\kappa \ge 0.772076n^2$, with the configuration $\chi=1.6$ as
the optimum.

\begin{lemma}
    \label{lem:lo_param_landscape_monotone}
     Consider ParamRLS-F tuning the \EA for LO with $\chi \in \{0.1, 0.2, \ldots, 2.9, 3.0\}$ (i.\,e. $d=10, \phi=3$). For all
    cutoff times $\kappa \ge 0.772076n^2$, and for the \EAc[a] and
    the \EAc[b], with either $0.1 \le b < a \le 1.6$ or $1.6 \le a < b \le 3.0$,
    the \EAc[a] wins a ParamRLS-F comparison against the \EAc[b] \wop
\end{lemma}

Lemma~\ref{lem:lo_param_landscape_monotone} proves that, with overwhelming
probability, for all cutoff times $\kappa \ge 0.772076n^2$ the configuration
$\chi=1.6$ wins a comparison in ParamRLS-F against any other configuration, and
also that in a ParamRLS-F comparison between two configurations both on the
same side of the configuration $\chi=1.6$, the configuration closer to $\chi=1.6$
wins the comparison \wop Having proved that the parameter landscape is unimodal, we can now
finally prove Theorem~\ref{thm:prlsf_can_tune_for_lo}.

\begin{proof}[Proof of Theorem~\ref{thm:prlsf_can_tune_for_lo}]
    We proceed in the same manner to the proof of
    Theorem~\ref{thm:prlsf_can_tune_for_ridge}. We pessimistically assume that
    the active parameter value is initialised as far away from the optimal
    parameter value as possible. The initial distance is clearly bounded by $d\phi$.

    Given a comparison of a pair of configurations which are both on the same
    side of the optimal configuration, let us call the configuration with a
    value of $\chi$ closer to the optimum the `better' configuration, and the
    other configuration the `worse' configuration.
    Lemma~\ref{lem:lo_param_landscape_monotone} tells us that in a comparison
    between any pair of configurations which are both on the same side of the
    optimum, the better configuration wins \wop Let
    us assume that the better configuration always beats the worse
    configuration. Since with the local search operator $\pm\{1\}$ the tuner will
    mutate the current configuration to one closer to the optimum, the tuner
    take a step towards the optimum with probability $1/2$. With the remaining probability, the active parameter will remain the same. The expected time to move closer to the optimum is thus~$2$. Since the tuner
    needs to take at most $d\phi$ steps towards the optimal configuration in this case, this
    implies that $\E[T] \le 2d\phi$. In the overwhelmingly unlikely
    event that the worse configuration wins a comparison
    then we restart the argument.
    Therefore, $\E[T] \le 2d\phi +
    \exp(-\Omega(n^{\varepsilon}))$ for some positive constant $\varepsilon$ from the definition of overwhelming probabilities.

    Using a Chernoff bound to count the number of times that the tuner takes a
    step towards the optimal configuration proves that, with overwhelming
    probability, $n^{\varepsilon}$ comparisons, for any positive constant
    $\varepsilon$, suffice for ParamRLS-F to set the active parameter value to
    the optimum. By the union bound, the value of the active
    parameter remains at the optimum \wop once it has been found, since there are polynomially many comparisons. This implies that,
    \wop, the tuner returns the optimal configuration
    for the cutoff time if run for at least this many comparisons.
\end{proof}

\subsection{Analysis of Optimisation Time-Based Comparisons when Tuning for {\scshape LO}}

As in Section~\ref{sec:opt_time_based_blind_ridge}, we prove here that \wop any configurator that uses the Optimisation time performance metric is
unable to tune the \EA for LO if $\kappa \le
0.772075n^2$.

\begin{theorem}
    \label{thm:prlst_fails_small_cutoffs_lo}
    Consider any configurator using the Optimisation time performance metric tuning the \EA for LO for any
    positive constant $\phi$ and discretisation factor~$d$. If the cutoff time for each run is never allowed to exceed $\kappa \le 0.772075n^2$ then, after any polynomial number of comparisons and runs per comparison, the configurator is blind.
\end{theorem}

To prove
Theorem~\ref{thm:prlst_fails_small_cutoffs_lo} we first show that, \wop, no
configuration of the \EA 
here reaches the optimum of
LO within this cutoff time.

\begin{lemma}
    \label{lem:lo_opt_time_lb}
    For all configurations $\chi \in \{0.1, 0.2, \ldots, 2.9, 3.0\}$, the \EAc does not reach the optimum of LO within $0.772075n^2$
    iterations, \wop
\end{lemma}

\begin{proof}[Proof of Lemma~\ref{lem:lo_opt_time_lb}]
    After $772075$ periods of length $n^2/1000000$ (that is, with
    $\psi=1000000$) we observe that the value $c_{u,\chi,i}$ for all $\chi \in
    \{0.1, 0.2, \ldots, 2.9, 3.0\}$ is less than 1. This implies that, with
    overwhelming probability, after $0.772075n^2$ iterations, no configuration
    has found the optimum of LO.
\end{proof}

Using Lemma~\ref{lem:lo_opt_time_lb}, we are now able to prove
Theorem~\ref{thm:prlst_fails_small_cutoffs_lo}.

\begin{proof}[Proof of Theorem~\ref{thm:prlst_fails_small_cutoffs_lo}]
    Since by Lemma~\ref{lem:lo_opt_time_lb} we know that, with overwhelming
    probability, no configuration finds the optimum of LO
    within $0.772075n^2$ iterations, then we argue that the result follows for
    the same reasons as in the proof of
    Theorem~\ref{thm:prlst_fails_small_cutoffs_rdg}.
\end{proof}

\section{Conclusions}
Recent experimental work has provided evidence that the algorithm configuration search landscapes for various NP-hard problems are more benign
than in worst-case scenarios. In this paper we rigorously proved that this is the case for the parameter landscape induced by the standard bit mutation (SBM) operator, used in evolutionary computation, for the optimisation of two standard benchmark problem classes, {\scshape Ridge} and {\scshape LeadingOnes}. In particular we have proved that the parameter landscape for both problems is largely unimodal. This effectively allows gradient-following algorithm configurators, including ParamRLS, to efficiently identify optimal mutation rates for both problems.

 To the best of our knowledge, the only other time complexity analysis of algorithm configurators for specific problems is our earlier work~\cite{paper:impact_of_cutoff_time_rlsk}, where we considered ParamRLS to tune the neighbourhood size of a more simple stochastic local search algorithm.
 This analysis pointed out that using the best identified fitness as performance measure (i.e., ParamRLS-F), rather than the optimisation time (i.e., ParamRLS-T), allows us to identify the optimal parameter value with considerably smaller cutoff times i.e., more efficiently.
 Our analysis reveals that this insight is also true for the much more sophisticated parameter landscape of the global mutation operator SBM.
 In particular, we proved for a wide range of cutoff times that ParamRLS-F tuning for {\scshape LeadingOnes} identifies that the smaller the cutoff time, the higher is the optimal mutation rate. For almost every given cutoff time, the optimal mutation rate for that cutoff time is returned efficiently, with overwhelming probability.
 Conversely, any algorithm configurator using optimisation time as performance metric is blind when using cutoff times that are smaller than the expected optimisation time of the optimal configuration.



\paragraph{Acknowledgements} This work was supported by the EPSRC under grant EP/M004252/1.

\bibliographystyle{ACM-Reference-Format}
\bibliography{Bibliography-File}


\begin{thebibliography}{21}


\ifx \showCODEN    \undefined \def \showCODEN     #1{\unskip}     \fi
\ifx \showDOI      \undefined \def \showDOI       #1{#1}\fi
\ifx \showISBNx    \undefined \def \showISBNx     #1{\unskip}     \fi
\ifx \showISBNxiii \undefined \def \showISBNxiii  #1{\unskip}     \fi
\ifx \showISSN     \undefined \def \showISSN      #1{\unskip}     \fi
\ifx \showLCCN     \undefined \def \showLCCN      #1{\unskip}     \fi
\ifx \shownote     \undefined \def \shownote      #1{#1}          \fi
\ifx \showarticletitle \undefined \def \showarticletitle #1{#1}   \fi
\ifx \showURL      \undefined \def \showURL       {\relax}        \fi
\providecommand\bibfield[2]{#2}
\providecommand\bibinfo[2]{#2}
\providecommand\natexlab[1]{#1}
\providecommand\showeprint[2][]{arXiv:#2}

\bibitem[\protect\citeauthoryear{Bartz-Beielstein, Lasarczyk, and
  Preu{\ss}}{Bartz-Beielstein et~al\mbox{.}}{2010}]%
        {paper:SPOT}
\bibfield{author}{\bibinfo{person}{Thomas Bartz-Beielstein},
  \bibinfo{person}{Christian Lasarczyk}, {and} \bibinfo{person}{Mike
  Preu{\ss}}.} \bibinfo{year}{2010}\natexlab{}.
\newblock \showarticletitle{The sequential parameter optimization toolbox}.
\newblock In \bibinfo{booktitle}{\emph{Experimental methods for the analysis of
  optimization algorithms}}. \bibinfo{publisher}{Springer},
  \bibinfo{pages}{337--362}.
\newblock


\bibitem[\protect\citeauthoryear{B{\"o}ttcher, Doerr, and Neumann}{B{\"o}ttcher
  et~al\mbox{.}}{2010}]%
        {paper:opt_mut_rates_1p1_LO}
\bibfield{author}{\bibinfo{person}{S{\"u}ntje B{\"o}ttcher},
  \bibinfo{person}{Benjamin Doerr}, {and} \bibinfo{person}{Frank Neumann}.}
  \bibinfo{year}{2010}\natexlab{}.
\newblock \showarticletitle{Optimal fixed and adaptive mutation rates for the
  {LeadingOnes} problem}.
\newblock \bibinfo{journal}{\emph{Parallel Problem Solving from Nature, PPSN
  XI}}, \bibinfo{pages}{1--10}.
\newblock


\bibitem[\protect\citeauthoryear{Burke, Gendreau, Hyde, Kendall, Ochoa,
  {\"O}zcan, and Qu}{Burke et~al\mbox{.}}{2013}]%
        {BurkeEtAl2013}
\bibfield{author}{\bibinfo{person}{Edmund~K. Burke}, \bibinfo{person}{Michel
  Gendreau}, \bibinfo{person}{Matthew Hyde}, \bibinfo{person}{Graham Kendall},
  \bibinfo{person}{Gabriela Ochoa}, \bibinfo{person}{Ender {\"O}zcan}, {and}
  \bibinfo{person}{Rong Qu}.} \bibinfo{year}{2013}\natexlab{}.
\newblock \showarticletitle{Hyper-heuristics: A survey of the state of the
  art}.
\newblock \bibinfo{journal}{\emph{Journal of the Operational Research Society}}
  \bibinfo{volume}{64}, \bibinfo{number}{12} (\bibinfo{year}{2013}),
  \bibinfo{pages}{1695--1724}.
\newblock


\bibitem[\protect\citeauthoryear{Doerr}{Doerr}{2020}]%
        {chap:prob_tools_analysis_rand_opt_heurs}
\bibfield{author}{\bibinfo{person}{Benjamin Doerr}.}
  \bibinfo{year}{2020}\natexlab{}.
\newblock \bibinfo{booktitle}{\emph{Probabilistic Tools for the Analysis of
  Randomized Optimization Heuristics}}.
\newblock \bibinfo{publisher}{Springer International Publishing},
  \bibinfo{pages}{1--87}.
\newblock


\bibitem[\protect\citeauthoryear{Droste, Jansen, Tinnefeld, and Wegener}{Droste
  et~al\mbox{.}}{2002}]%
        {DrosteJansenWegener2002b}
\bibfield{author}{\bibinfo{person}{Stefan Droste}, \bibinfo{person}{Thomas
  Jansen}, \bibinfo{person}{Karsten Tinnefeld}, {and} \bibinfo{person}{Ingo
  Wegener}.} \bibinfo{year}{2002}\natexlab{}.
\newblock \showarticletitle{A New Framework for the Valuation of Algorithms for
  Black-Box Optimization}. In \bibinfo{booktitle}{\emph{Proceedings of the
  Seventh Workshop on Foundations of Genetic Algorithms}}.
  \bibinfo{publisher}{Morgan Kaufmann}, \bibinfo{pages}{253--270}.
\newblock


\bibitem[\protect\citeauthoryear{Hall, Oliveto, and Sudholt}{Hall
  et~al\mbox{.}}{2019}]%
        {paper:impact_of_cutoff_time_rlsk}
\bibfield{author}{\bibinfo{person}{George~T. Hall}, \bibinfo{person}{Pietro~S.
  Oliveto}, {and} \bibinfo{person}{Dirk Sudholt}.}
  \bibinfo{year}{2019}\natexlab{}.
\newblock \showarticletitle{On the Impact of the Cutoff Time on the Performance
  of Algorithm Configurators}. In \bibinfo{booktitle}{\emph{Proceedings of the
  Genetic and Evolutionary Computation Conference}}
  \emph{(\bibinfo{series}{GECCO '19})}. \bibinfo{publisher}{ACM},
  \bibinfo{pages}{907--915}.
\newblock


\bibitem[\protect\citeauthoryear{He and Yao}{He and Yao}{2001}]%
        {He2001}
\bibfield{author}{\bibinfo{person}{Jun He} {and} \bibinfo{person}{Xin Yao}.}
  \bibinfo{year}{2001}\natexlab{}.
\newblock \showarticletitle{Drift analysis and average time complexity of
  evolutionary algorithms}.
\newblock \bibinfo{journal}{\emph{Artificial Intelligence}}
  \bibinfo{volume}{127}, \bibinfo{number}{1} (\bibinfo{year}{2001}),
  \bibinfo{pages}{57--85}.
\newblock


\bibitem[\protect\citeauthoryear{Hutter, Hoos, and Leyton-Brown}{Hutter
  et~al\mbox{.}}{2009}]%
        {paper:paramILS}
\bibfield{author}{\bibinfo{person}{Frank Hutter}, \bibinfo{person}{Holger~H.
  Hoos}, {and} \bibinfo{person}{Kevin Leyton-Brown}.}
  \bibinfo{year}{2009}\natexlab{}.
\newblock \showarticletitle{Param{ILS}: an automatic algorithm configuration
  framework}.
\newblock \bibinfo{journal}{\emph{Journal of Artificial Intelligence Research}}
  \bibinfo{volume}{36}, \bibinfo{number}{1} (\bibinfo{year}{2009}),
  \bibinfo{pages}{267--306}.
\newblock


\bibitem[\protect\citeauthoryear{Hutter, Hoos, and Leyton-Brown}{Hutter
  et~al\mbox{.}}{2011}]%
        {paper:ROAR_and_SMAC}
\bibfield{author}{\bibinfo{person}{Frank Hutter}, \bibinfo{person}{Holger~H.
  Hoos}, {and} \bibinfo{person}{Kevin Leyton-Brown}.}
  \bibinfo{year}{2011}\natexlab{}.
\newblock \showarticletitle{Sequential model-based optimization for general
  algorithm configuration}. In \bibinfo{booktitle}{\emph{International
  Conference on Learning and Intelligent Optimization}}. Springer,
  \bibinfo{pages}{507--523}.
\newblock


\bibitem[\protect\citeauthoryear{Jansen and Zarges}{Jansen and Zarges}{2014}]%
        {Jansen2014}
\bibfield{author}{\bibinfo{person}{Thomas Jansen} {and}
  \bibinfo{person}{Christine Zarges}.} \bibinfo{year}{2014}\natexlab{}.
\newblock \showarticletitle{Performance analysis of randomised search
  heuristics operating with a fixed budget}.
\newblock \bibinfo{journal}{\emph{Theoretical Computer Science}}
  \bibinfo{volume}{545} (\bibinfo{year}{2014}), \bibinfo{pages}{39--58}.
\newblock


\bibitem[\protect\citeauthoryear{KhudaBukhsh, Xu, Hoos, and
  Leyton{-}Brown}{KhudaBukhsh et~al\mbox{.}}{2016}]%
        {SATenstein}
\bibfield{author}{\bibinfo{person}{Ashiqur~R. KhudaBukhsh},
  \bibinfo{person}{Lin Xu}, \bibinfo{person}{Holger~H. Hoos}, {and}
  \bibinfo{person}{Kevin Leyton{-}Brown}.} \bibinfo{year}{2016}\natexlab{}.
\newblock \showarticletitle{{SAT}enstein: Automatically building local search
  {SAT} solvers from components}.
\newblock \bibinfo{journal}{\emph{Artificial Intelligence}}
  \bibinfo{volume}{232} (\bibinfo{year}{2016}), \bibinfo{pages}{20--42}.
\newblock


\bibitem[\protect\citeauthoryear{Kleinberg, Leyton-Brown, and Lucier}{Kleinberg
  et~al\mbox{.}}{2017}]%
        {paper:efficiency_through_procrastination}
\bibfield{author}{\bibinfo{person}{Robert Kleinberg}, \bibinfo{person}{Kevin
  Leyton-Brown}, {and} \bibinfo{person}{Brendan Lucier}.}
  \bibinfo{year}{2017}\natexlab{}.
\newblock \showarticletitle{Efficiency Through Procrastination: Approximately
  Optimal Algorithm Configuration with Runtime Guarantees}. In
  \bibinfo{booktitle}{\emph{Proceedings of the 26th International Joint
  Conference on Artificial Intelligence}} \emph{(\bibinfo{series}{IJCAI'17})}.
  \bibinfo{publisher}{AAAI Press}, \bibinfo{pages}{2023--2031}.
\newblock
\showISBNx{978-0-9992411-0-3}


\bibitem[\protect\citeauthoryear{Kleinberg, Leyton-Brown, Lucier, and
  Graham}{Kleinberg et~al\mbox{.}}{2019}]%
        {paper:spc}
\bibfield{author}{\bibinfo{person}{Robert Kleinberg}, \bibinfo{person}{Kevin
  Leyton-Brown}, \bibinfo{person}{Brendan Lucier}, {and} \bibinfo{person}{Devon
  Graham}.} \bibinfo{year}{2019}\natexlab{}.
\newblock \showarticletitle{Procrastinating with Confidence: Near-Optimal,
  Anytime, Adaptive Algorithm Configuration}.
\newblock In \bibinfo{booktitle}{\emph{Advances in Neural Information
  Processing Systems 32}}. \bibinfo{publisher}{Curran Associates, Inc.},
  \bibinfo{pages}{8881--8891}.
\newblock


\bibitem[\protect\citeauthoryear{K{\"{o}}tzing}{K{\"{o}}tzing}{2016}]%
        {DBLP:journals/algorithmica/Kotzing16}
\bibfield{author}{\bibinfo{person}{Timo K{\"{o}}tzing}.}
  \bibinfo{year}{2016}\natexlab{}.
\newblock \showarticletitle{Concentration of First Hitting Times Under Additive
  Drift}.
\newblock \bibinfo{journal}{\emph{Algorithmica}} \bibinfo{volume}{75},
  \bibinfo{number}{3} (\bibinfo{year}{2016}), \bibinfo{pages}{490--506}.
\newblock


\bibitem[\protect\citeauthoryear{Lehre and Witt}{Lehre and Witt}{2012}]%
        {Lehre2012}
\bibfield{author}{\bibinfo{person}{Per~Kristian Lehre} {and}
  \bibinfo{person}{Carsten Witt}.} \bibinfo{year}{2012}\natexlab{}.
\newblock \showarticletitle{Black-Box Search by Unbiased Variation}.
\newblock \bibinfo{journal}{\emph{Algorithmica}} \bibinfo{volume}{64},
  \bibinfo{number}{4} (\bibinfo{year}{2012}), \bibinfo{pages}{623--642}.
\newblock


\bibitem[\protect\citeauthoryear{Lissovoi, Oliveto, and Warwicker}{Lissovoi
  et~al\mbox{.}}{2019}]%
        {LissovoiEtAl2019}
\bibfield{author}{\bibinfo{person}{Andrei Lissovoi}, \bibinfo{person}{Pietro~S.
  Oliveto}, {and} \bibinfo{person}{John~Alasdair Warwicker}.}
  \bibinfo{year}{2019}\natexlab{}.
\newblock \showarticletitle{On the Time Complexity of Algorithm Selection
  Hyper-Heuristics for Multimodal Optimisation}. In
  \bibinfo{booktitle}{\emph{Thirty-Third AAAI Conference on Artificial
  Intelligence}} \emph{(\bibinfo{series}{AAAI `19})}. \bibinfo{publisher}{AAAI
  Press}, \bibinfo{pages}{2322--2329}.
\newblock


\bibitem[\protect\citeauthoryear{L{\'o}pez-Ib{\'a}{\~n}ez, Dubois-Lacoste,
  C{\'a}ceres, Birattari, and St{\"u}tzle}{L{\'o}pez-Ib{\'a}{\~n}ez
  et~al\mbox{.}}{2016}]%
        {paper:irace}
\bibfield{author}{\bibinfo{person}{Manuel L{\'o}pez-Ib{\'a}{\~n}ez},
  \bibinfo{person}{J{\'e}r{\'e}mie Dubois-Lacoste},
  \bibinfo{person}{Leslie~P{\'e}rez C{\'a}ceres}, \bibinfo{person}{Mauro
  Birattari}, {and} \bibinfo{person}{Thomas St{\"u}tzle}.}
  \bibinfo{year}{2016}\natexlab{}.
\newblock \showarticletitle{The irace package: Iterated racing for automatic
  algorithm configuration}.
\newblock \bibinfo{journal}{\emph{Operations Research Perspectives}}
  \bibinfo{volume}{3} (\bibinfo{year}{2016}), \bibinfo{pages}{43--58}.
\newblock


\bibitem[\protect\citeauthoryear{Minton}{Minton}{1993}]%
        {paper:integrating_heuristics_constraint_satisfaction_probs}
\bibfield{author}{\bibinfo{person}{Steven Minton}.}
  \bibinfo{year}{1993}\natexlab{}.
\newblock \showarticletitle{Integrating Heuristics for Constraint Satisfaction
  Problems: A Case Study}. In \bibinfo{booktitle}{\emph{AAAI 1993}}.
  \bibinfo{publisher}{AAAI Press}, \bibinfo{pages}{120--126}.
\newblock


\bibitem[\protect\citeauthoryear{Pushak and Hoos}{Pushak and Hoos}{2018}]%
        {paper:algo_config_landscapes_benign}
\bibfield{author}{\bibinfo{person}{Yasha Pushak} {and}
  \bibinfo{person}{Holger~H. Hoos}.} \bibinfo{year}{2018}\natexlab{}.
\newblock \showarticletitle{Algorithm Configuration Landscapes: More Benign
  Than Expected?}. In \bibinfo{booktitle}{\emph{Parallel Problem Solving from
  Nature -- PPSN XV}}. \bibinfo{publisher}{Springer International Publishing},
  \bibinfo{pages}{271--283}.
\newblock


\bibitem[\protect\citeauthoryear{Weisz, Gyorgy, and Szepesv{\'a}ri}{Weisz
  et~al\mbox{.}}{2018}]%
        {paper:leaps_and_bounds}
\bibfield{author}{\bibinfo{person}{Gell{\'e}rt Weisz}, \bibinfo{person}{Andras
  Gyorgy}, {and} \bibinfo{person}{Csaba Szepesv{\'a}ri}.}
  \bibinfo{year}{2018}\natexlab{}.
\newblock \showarticletitle{{Leaps\-And\-Bounds}: A Method for Approximately
  Optimal Algorithm Configuration}. In \bibinfo{booktitle}{\emph{Proceedings of
  the 35th International Conference on Machine Learning}}
  \emph{(\bibinfo{series}{Proceedings of Machine Learning Research})},
  Vol.~\bibinfo{volume}{80}. \bibinfo{publisher}{PMLR},
  \bibinfo{pages}{5257--5265}.
\newblock


\bibitem[\protect\citeauthoryear{Weisz, Gyorgy, and Szepesv{\'a}ri}{Weisz
  et~al\mbox{.}}{2019}]%
        {paper:caps_and_runs}
\bibfield{author}{\bibinfo{person}{Gell{\'e}rt Weisz}, \bibinfo{person}{Andras
  Gyorgy}, {and} \bibinfo{person}{Csaba Szepesv{\'a}ri}.}
  \bibinfo{year}{2019}\natexlab{}.
\newblock \showarticletitle{CapsAndRuns: An improved method for approximately
  optimal algorithm configuration}. In \bibinfo{booktitle}{\emph{Proceedings of
  the 36th International Conference on Machine Learning}}
  \emph{(\bibinfo{series}{Proceedings of Machine Learning Research})},
  Vol.~\bibinfo{volume}{97}. \bibinfo{publisher}{PMLR},
  \bibinfo{pages}{6707--6715}.
\newblock


\end{thebibliography}


\clearpage
\appendix

\section{Omitted Proofs from Section~\ref{sec:ridge} ({\scshape Ridge})}

\subsection{Proof of Lemma~\ref{lem:ridge_facts}.}

\begin{proof}[Proof of Lemma~\ref{lem:ridge_facts}]
    We split this proof into a section for each claim of the lemma. \\
    \textit{Proof of Lemma~\ref{lem:ridge_facts}(\ref{lem:exp_drift_1p1_a_ridge}).} \\
    By definition,
    \begin{align*}
        & \E[\Delta_{\chi}(x_{t}) \mid x_{t}, \text{\scshape Ridge}(x_t) = j < n] \\
        & = \sum_{i=1}^{n-j} (i \cdot \prob(\text{\scshape Ridge}(x_{t+1}) - \text{\scshape Ridge}(x_{t}) = i))
    \end{align*}
    Due to the nature of {\scshape Ridge}, we know that the current individual
    will be in the form $1^j0^{n-j}$. This means that, in order to improve by
    exactly $i$ in a single iteration, we must flip exactly the first $i$
    leading zeroes in the bit string. The probability of doing this is
    $(\chi/n)^i(1-\chi/n)^{n-i}$. This implies that
    \begin{align*}
        & \E[\Delta_{\chi}(x_{t}) \mid x_t, \text{\scshape Ridge}(x_t) = j < n]\\
        & = \sum_{i=1}^{n-j} \left( i \cdot \left( \frac{\chi}{n} \right)^i \left( 1 - \frac{\chi}{n} \right)^{n-i} \right).
    \end{align*}
    We can trivially lower-bound this sum by its first term:
    \[ \E[\Delta_{\chi}(x_{t}) \mid x_t, \text{\scshape Ridge}(x_t) = j < n] \ge \frac{\chi}{n} \left( 1 - \frac{\chi}{n} \right)^{n-1} \]
    We bound $\E[\Delta_{\chi}(x_{t}) \mid x_t, \text{\scshape Ridge}(x_t) = j < n]$ from above by observing
    that
    \begin{align*}
    & \sum_{i=1}^{n-j} \left( i \cdot \left( \frac{\chi}{n} \right)^i \left( 1 - \frac{\chi}{n} \right)^{n-i} \right)\\
     &= \left( 1 - \frac{\chi}{n} \right)^{n}
    \sum_{i=1}^{n-j}
    \left( i \cdot \left( \frac{\chi}{n-\chi} \right)^{i} \right) \\
        &\le \left( 1 - \frac{\chi}{n} \right)^{n} \sum_{i=1}^{\infty} \left( i \cdot \left( \frac{\chi}{n-\chi} \right)^{i} \right) \\
        &= \left( 1 - \frac{\chi}{n} \right)^{n} \cdot \frac{\chi(n-\chi)}{(n-2\chi)^{2}}
    \end{align*}
    and then observing that
    \begin{align*}
        &\left( 1 - \frac{\chi}{n} \right)^{n} \cdot \frac{\chi(n-\chi)}{(n-2\chi)^{2}}\\
        &= \frac{\chi}{n} \left( 1 - \frac{\chi}{n} \right)^{n-1} \cdot \left(
        \frac{n(n-\chi)}{(n-2\chi)^{2}}
        \left( 1 - \frac{\chi}{n} \right) \right)\\
        &= \frac{\chi}{n} \left( 1 - \frac{\chi}{n} \right)^{n-1} \cdot
        \frac{(n-\chi)^2}{(n-2\chi)^{2}}\\
        &= \frac{\chi}{n} \left( 1 - \frac{\chi}{n} \right)^{n-1} + O\left(\frac{1}{n^2}\right).\\
    \end{align*}

    \noindent\textit{Proof of Lemma~\ref{lem:ridge_facts}(\ref{lem:proof_c_1_optimal_ridge}).} \\
    By
    Lemma~\ref{lem:ridge_facts}(\ref{lem:exp_drift_1p1_a_ridge}),
    the expected drift of the \EAc[1] optimising {\scshape Ridge}, for every $x_t \neq 1^n$, is
    \[ \E[\Delta_{1}(x_{t}) \mid x_{t}, x_t \neq 1^n] \ge \frac{1}{n} \left( 1 - \frac{1}{n} \right)^{n-1}\]
    By the same lemma we have that
    \[ \E[\Delta_{\chi}(x_{t}) \mid x_{t}] \le \frac{\chi}{n} \left( 1 - \frac{\chi}{n} \right)^{n-1} + \Theta \left( \frac{1}{n^{2}} \right). \]
    Since the expression $\frac{\chi}{n} \left( 1 - \frac{\chi}{n} \right)^{n-1}$ is
    maximised for $\chi=1$ and is $\Theta(1/n)$ for all positive constants $\chi$, we
    have that, for sufficiently large $n$, the lower bound on the expected
    drift of the \EAc[1] is larger than the upper bound on the expected drift
    of the \EA with $\chi \neq 1$. Since these expressions for the drift do not
    depend on the fitness of the individual, we are able to determine the
    expected optimisation time of the \EA on {\scshape Ridge} using additive
    drift analysis \cite{He2001,DBLP:journals/algorithmica/Kotzing16}. Since all algorithms
    are initialised to $0^n$, a distance of~$n$ has to be bridged.
    Theorem~5 in~\cite{DBLP:journals/algorithmica/Kotzing16} states that the expected optimisation time is at most $n$ divided by a lower bound on the expected drift, and it is at least $n$ divided by an upper bound on the expected drift. Along with the first statement of this lemma, this implies that the expected optimisation time is at most $en^2$ for $\chi=1$ and at least $e^{\chi}/\chi \cdot n^2 - O(n)$, which is larger than $en^2$ for all $\chi \neq 1$ if $n$ is large enough.
    Hence, for large enough~$n$, the expected optimisation time for the \EAc[1] is less than the
    expected optimisation time for all other configurations with constant $\chi
    \neq 1$.
\end{proof}

\subsection{Proof of Lemma~\ref{lem:prob_1p1_a_beats_1p1_b}}

In order to prove this lemma, it is first necessary to prove a helper lemma.

\begin{lemma}
    \label{lem:prob_1p1_a_ahead_1p1_b}
    Assume that the \EAc[a] and the \EAc[b], with $a$ and $b$ two distinct non-negative
    constants such that $a/\exp(a) > b/\exp(b)$, are both initialised to $0^n$.
    Then with
    probability at least
    \[ 1-3\exp(-\Omega(\kappa/n)) \]
    the \EAc[a] has a higher fitness on {\scshape Ridge} than the \EAc[b] after
    $\kappa$ iterations or one of the considered algorithms has found a global optimum in the first $\kappa-1$ iterations.
\end{lemma}

\begin{proof}[Proof of Lemma~\ref{lem:prob_1p1_a_ahead_1p1_b}]
    Define $A_{t}$ and $B_{t}$ to be the fitness values of the \EAc[a] and
    the \EAc[b], respectively, after $t$ iterations on {\scshape Ridge}.
    Note that the probability of an improvement for the \EAc[a] is $\prob(A_{t+1} \ge A_t+1 \mid A_t < n) \ge a/n \cdot (1-a/n)^{n-1}$.
    We may drop the condition $A_t < n$ as if it is violated then an optimum has been found within the first $t$ steps. This means we assume that the number of improvements is not bounded. The fitness $A_\kappa$ is obviously at least as large as the number of improvements in the first $\kappa$ generations. Applying
    Chernoff bounds to the latter, for every constant $0 < \varepsilon < 1$,
    \[
        \prob\left(A_\kappa \le (1-\varepsilon)\kappa \cdot \frac{a}{n} \left(1 - \frac{a}{n}\right)^{n-1}\right) \le \exp(-\Omega(\kappa/n)).
    \]
    Using $\left(1 - \frac{a}{n}\right)^{n-1} = \left(1 - \frac{a}{n}\right)^{n-a} \left(1 - \frac{a}{n}\right)^{a-1}
    \ge e^{-a} \left(1 - \frac{a}{n}\right)^{a-1} \ge e^{-a} \left(1-O\left(\frac{1}{n}\right)\right)$ where the last step is trivial for $a \le 1$ and for $a > 1$ follows from Bernoulli's inequality,
    \begin{equation}
    \label{eq:upper-bound-on-A}
     A_\kappa \ge (1-\varepsilon)\kappa \cdot \frac{a e^{-a}}{n} \left(1 - O\left(\frac{1}{n}\right)\right)
    \end{equation}
    with probability $1-\exp(-\Omega(\kappa/n))$.

    We bound $B_\kappa$ from above in a similar fashion. However, we need to take into account the possibility of jumps that increase the fitness by more than~1. Note that, for all $i \in \mathbb{N}$,
    \begin{align*}
    & \prob(B_{t+1} = B_t + i \mid B_t) \le \left(\frac{b}{n}\right)^i \left(1 - \frac{b}{n}\right)^{n-i}\\
    =\;& \left(\frac{b}{n-b}\right)^i \left(1 - \frac{b}{n}\right)^n\\
    =\;& \left(\frac{b}{n-b}\right)^{i-1} \left(1 - \frac{b}{n-b}\right) \cdot \left(1 - \frac{b}{n}\right)^n \cdot \frac{b}{n-2b}.
    \end{align*}
    Hence $(B_{t+1}-B_t \mid B_t)$ has the same distribution as the convolution $X_{t+1}Y_{t+1}$ where $X_{t+1}$ is a Bernoulli random variable with parameter $p_X := \left(1 - \frac{b}{n}\right)^n \cdot \frac{b}{n-2b}$ and $Y_{t+1}$ is a geometric random variable with parameter $1-\frac{b}{n-b}$; in other words, $\prob(Y_{t+1} = i) =  \left(\frac{b}{n-b}\right)^{i-1} \left(1 - \frac{b}{n-b}\right)$.
    Intuitively, the $X$-variables can be seen as indicator variables signalling whether an improvement happens and the $Y$-variables correspond to the jump length in an improving generation.

    Applying Chernoff bounds to the variables $X_1, \dots, X_\kappa$, for every constant $0 < \varepsilon < 1$,
    \[
        \prob\left(\sum_{t=1}^\kappa X_t \ge (1+\varepsilon)\kappa p_X\right) \le \exp(-\Omega(\kappa/n)).
    \]
    Assuming $\sum_{t=1}^\kappa X_t \le n_X := (1+\varepsilon)\kappa p_X$, we bound the contribution of up to this number of $Y$-variables whose corresponding $X$-variable is 1, using Chernoff bounds for geometric random variables. For ease of notation, we rename these variables $Y_1, \dots, Y_{n_X}$.
    This yields
    \begin{align*}
         \prob\left(\sum_{t=1}^{n_X} Y_t \ge (1+\varepsilon) n_X \cdot \frac{n-b}{n-2b}\right)
        \le\;& \exp(-\Omega(n_X))\\
        =\;& \exp(-\Omega(\kappa/n)).
    \end{align*}
    Together, with probability at least $1-2\exp(-\Omega(\kappa/n))$
    \begin{align}
    B_\kappa \le \sum_{t=1}^\kappa X_t Y_t \le\;& (1+\varepsilon) n_X \cdot \frac{n-b}{n-2b}\notag\\
    =\;& (1+\varepsilon)^2 \kappa \left(1 - \frac{b}{n}\right)^n \cdot \frac{b}{n-2b} \cdot \frac{n-b}{n-2b}\notag\\
    \le\;& (1+\varepsilon)^2 \kappa  \frac{b e^{-b}}{n-2b} \cdot \frac{n-b}{n-2b}\notag\\
    \le\;& (1+\varepsilon)^2 \kappa  \frac{b e^{-b}}{n} \cdot \left(1 + O\left(\frac{1}{n}\right)\right).\label{eq:lower-bound-on-B}
    \end{align}
    Since $ae^{-a} > be^{-b}$ we can choose $\varepsilon$ small enough such that $(1-\varepsilon)ae^{-b} (1-O(1/n)) > (1+\varepsilon)^2be^{-b} (1+O(1/n))$, for large enough~$n$. Then the lower bound for $A_\kappa$ from~\eqref{eq:upper-bound-on-A} and the lower bound for $B_\kappa$ from~\eqref{eq:lower-bound-on-B} imply that
    $A_\kappa > B_\kappa$ with probability at least $1 - 3\exp(-\Omega(\kappa/n))$. 
\end{proof}

We are now able to prove Lemma~\ref{lem:prob_1p1_a_beats_1p1_b}.

\begin{proof}[Proof of Lemma~\ref{lem:prob_1p1_a_beats_1p1_b}]
    It is easy to show that the probability that either algorithm reaches the optimum within the first $n^2$ iterations is $\exp(-\Omega(n))$ (cf. Lemma~\ref{lem:ridge_opt_time_lower_bound} presented below).
    If $\kappa \le n^2$ then the statement follows from
    Lemma~\ref{lem:prob_1p1_a_ahead_1p1_b} and a union bound over failure
    probabilities from Lemma~\ref{lem:prob_1p1_a_ahead_1p1_b} and the aforementioned term of $\exp(-\Omega(n))$.

    If $\kappa > n^2$ we argue that \EAc[b] only wins the comparison if either it finds the optimum before \EAc[a] or if it is ahead of \EAc[a] after $\kappa$ iterations.
    A necessary condition for this union of events is that \EAc[b] is ahead of \EAc[a] during some point in time~$t$ with $n^2 \le t \le \kappa$. Applying Lemma~\ref{lem:prob_1p1_a_ahead_1p1_b} for all such~$t$ and taking a union bound, the probability that this happens is at most
    \[
    \sum_{t=n^2}^\kappa
    3\exp(-\Omega(t/n)) \le (\kappa-1) 3\exp(-\Omega(n)).
    \]
    The claim then follows since the sum of all failure probabilities is at
    most
    $(\kappa-1) 3\exp(-\Omega(n)) + \exp(-\Omega(n)) \le \kappa\exp(-\Omega(n))$.

    For the remark at the end of the lemma, note that the expression
    $\chi e^{-\chi}$ is maximised for $\chi=1$ and decreases monotonically either side of
    this point. The claim therefore holds since $a$ is closer to 1 than $b$ by
    assumption.
%
%
\end{proof}

\begin{proof}[Proof of Lemma~\ref{lem:ridge_opt_time_lower_bound}]
    By the progress bounds established in the proof of Lemma~\ref{lem:prob_1p1_a_ahead_1p1_b}, the
    fitness of the individual in the \EA after $\kappa := (1-\varepsilon')en^2$
    iterations on {\scshape Ridge} is at most
    \begin{align*}
        (1+\varepsilon'')^2 (1-\varepsilon)n \cdot \left(1 + O\left(\frac{1}{n}\right)\right) &\le (1+\varepsilon'')^3 (1-\varepsilon)n \\
        &\le (1-\varepsilon')n
    \end{align*}
    for suitably chosen positive constants $\varepsilon'$ and $\varepsilon''$,
    and large enough $n$. Hence, for cutoff times of $\kappa \le
    (1-\varepsilon)n^2$, the \EA has not reached the optimum of {\scshape
    Ridge}, with probability at least $1-2\exp(-\Omega(n)) \allowbreak= \allowbreak1-\exp(-\Omega(n))$.


\end{proof}

\section{Omitted Proofs from Section~\ref{sec:lo} ({\scshape LeadingOnes})}

\subsection{Proof of Lemma~\ref{lem:lo_progress_bounds}}

For this lemma, it is necessary to derive several helper lemmas. These can be
combined to proved Lemma~\ref{lem:lo_progress_bounds}. We begin by deriving an
upper bound on the number of iterations in which the \EA increases the fitness
of the individual (iterations which we call \emph{improvements}) in a period of
$n^2/\psi$ iterations, for some positive constant $\psi$.

\begin{lemma}
    \label{lem:1p1ea_lo_num_improvements_ub}
    With overwhelming probability, the \EA optimising LO,
    starting at a fitness of $i$, makes at most
    \[ (1+n^{-1/4}) \frac{\chi n}{\psi \cdot \exp \left( \frac{\chi i}{n} \right)} \]
    improvements in a period of length $n^{2}/\psi$.
\end{lemma}

\begin{proof}[Proof of Lemma~\ref{lem:1p1ea_lo_num_improvements_ub}]
    The probability of the \EA making an improvement is equal to the probability that
    it flips the first leading zero and does not flip any of the leading ones.
    Let us assume that the current fitness of the individual is $j$. Then the
    probability of making an improvement is therefore
    \[ \prob(\text{LO}(x_{t+1}) > \text{LO}(x_{t}) \mid \text{{\scshape LO}}(x_{t}) = j) = \frac{\chi}{n} \left( 1 - \frac{\chi}{n} \right)^{j} \]
    Since $j \ge i$ by assumption, we have
    \[ \prob(\text{LO}(x_{t+1}) > \text{LO}(x_{t}) \mid \text{{\scshape LO}}(x_{t}) = j) \le \frac{\chi}{n} \left( 1 - \frac{\chi}{n} \right)^{i} \]
    Using \cite[Corollary~1.4.4]{chap:prob_tools_analysis_rand_opt_heurs},
    the quantity $(1 - (\chi/n))^{i}$ can be bounded from above by $\exp(-\chi i/n)$.
    Hence
    \[ \prob(\text{LO}(x_{t+1}) > \text{LO}(x_{t}) \mid \text{{\scshape LO}}(x_{t}) = j) \le \frac{\chi}{n} \exp \left( - \frac{\chi i}{n} \right) \]
    We now use a Chernoff bound to bound the number of improvements made in a single
    period of length $n^{2}/\psi$, for some positive constant $\psi$. Let the
    random variable $X_{k}$ equal 1 if and only if an improvement occurs in iteration
    $k$. Otherwise let it equal 0. Now define the random variable $Y^{\psi}$ as
    $Y^{\psi} := \sum_{k=0}^{(n^2/\psi)-1} X_{k}$. That is, $Y^{\psi}$ counts the
    number of improvements which occur in a period of length $n^2/\psi$. Since $X_{k}$ is an indicator variable we have that
    \[ \E[Y^{\psi}] \le \frac{\chi n}{\psi \exp \left( \frac{\chi i}{n} \right)} \]
    Let us now optimistically assume that $\E[Y^{\psi}]$ is equal to
    this upper bound. Then, by a standard Chernoff bound, we derive that
    \begin{align*}
        & \prob \left( Y^{\psi} \ge (1+n^{-1/4}) \frac{\chi n}{\psi \exp \left( \frac{\chi i}{n} \right)} \right)\\
        & \le \exp \left( - \frac{n^{-1/2} \chi n}{3\psi \exp \left( \frac{\chi i}{n} \right)} \right)\\
        & \le \exp \left( - \frac{n^{-1/2} \chi n}{3\psi e^\chi} \right) =
    \exp(-\Omega(n^{1/2})). \qedhere
    \end{align*}
\end{proof}

We now derive a corresponding lower bound on the number of improvements in a
period of the same length. We show that we make at least $i$ improvements
during the period or we exceed a fitness of $i$ by the end of the period.

\begin{lemma}
    \label{lem:1p1ea_lo_num_improvements_lb}
    Assume that the \EA optimising LO currently has a
    fitness of $j$. Fix a value of $i \le n$ and consider a period of length $n^2/\psi$ (for any positive
    constant~$\psi$).
    Then, \wop, during said period the \EA makes at least
    \[ (1-n^{-1/4}) \frac{\chi n}{\psi \cdot \exp \left( \frac{\chi i}{n-\chi} \right)} \]
    improvements, or the \EA
    exceeds a fitness of $i$.
\end{lemma}

\begin{proof}[Proof of Lemma~\ref{lem:1p1ea_lo_num_improvements_lb}]
    Assuming that the current fitness of
    the individual is $j$, the probability of making an improvement is
    \[ \prob(\text{LO}(x_{t+1}) > \text{LO}(x_{t}) \mid \text{{\scshape LO}}(x_{t}) = j) = \frac{\chi}{n} \left( 1 - \frac{\chi}{n} \right)^{j} \]
    Since $j \le i$ by assumption, we have
    \[ \prob(\text{LO}(x_{t+1}) > \text{LO}(x_{t}) \mid \text{{\scshape LO}}(x_{t}) = j) \ge \frac{\chi}{n} \left( 1 - \frac{\chi}{n} \right)^{i} \]
    Using \cite[Corollary~1.4.4]{chap:prob_tools_analysis_rand_opt_heurs},
    the quantity $(1 - (\chi/n))^{i}$ can be bounded from below by
    $\exp(-\chi i/(n-\chi))$. This implies that
    \begin{align*}
    & \prob(\text{LO}(x_{t+1}) > \text{LO}(x_{t}) \mid \text{{\scshape LO}}(x_{t}) \le i)\\
    & \ge \frac{\chi}{n} \exp \left( - \frac{\chi i}{n-\chi} \right).
    \end{align*}
    We now use a Chernoff bound to lower-bound the number of improvements made in a single
    period of length $n^{2}/\psi$, for some positive constant $\psi$.
    Let the
    indicator random variable $X_{k}$ attain value 1 with probability
    $\frac{\chi}{n} \exp \left( - \frac{\chi i}{n-\chi} \right)$. Now define the random variable $Y^{\psi}$ as
    $Y^{\psi} := \sum_{k=0}^{(n^2/\psi)-1} X_{k}$. That is, $Y^{\psi}$ is stochastically dominated by the
    number of improvements which occur in a period of length $n^2/\psi$, unless a fitness larger than~$i$ is reached. Note that the event of exceeding fitness~$i$
    is
    included in the event whose probability we aim to bound from below.
    Since $X_{k}$ is an indicator variable we have that
    \[ \E[Y^{\psi}] \ge \frac{\chi n}{\psi \cdot \exp \left( \frac{\chi i}{n-\chi} \right)} \]
    By a standard Chernoff bound, we derive that
    \begin{align*}
     & \prob \left( Y^{\psi} \le (1-n^{-1/4}) \frac{\chi n}{\psi \cdot \exp \left( \frac{\chi i}{n-\chi} \right)} \right)\\
      & \le \exp \left( - \frac{n^{-1/2} \chi n}{2\psi \cdot \exp
    \left( \frac{\chi i}{n-\chi} \right)} \right) = \exp(-\Omega(n^{1/2})).\qedhere
    \end{align*}
\end{proof}

Associated with each improvement are a number of \emph{free riders}. These are the
consecutive 1 bits immediately following the former leading 0 which has just
been flipped.  Therefore the total fitness gained in each improvement is one more than
the number of associated free riders. We first show an upper bound on the number of free riders
encountered during a period of length $n^2/\psi$.

\begin{lemma}
    \label{lem:num_free_riders_ub}
    Consider $\ell$ improving steps of the \EA optimising LO. Then with probability $1-\exp(-\Omega(\ell^{1/2}))$, it gains
    at most $(1+\ell^{-1/4})\ell$ in fitness from free riders in total.
\end{lemma}

\begin{proof}[Proof of Lemma~\ref{lem:num_free_riders_ub}]
    For this proof, we assume that the algorithm is operating on an infinite
    bit string. This is a valid assumption since we are deriving an upper bound
    on the number of free riders which occur during the optimisation process,
    and the number which occurs in reality is strictly less than the number
    which is possible on an infinite bit string. It is well known
    that the distribution of the ones and zeroes following the leading zero is
    uniformly at random, since this section of the bit string has no effect on
    the fitness of the individual (see Lemma~1 in~\cite{Lehre2012} for a formal proof). This fact and our infinite bit string
    assumption allow us to use a geometric random variable $X_k$ with parameter~$1/2$ to count number of
    free riders in the $k$-th improvement. Define $X^{\ell} :=
    \sum_{k=1}^{\ell} X_{k}$ as the total number of free riders encountered over
    the $\ell$ improvements. Using a Chernoff bound for sums of geometric random variables,
    \cite[Theorem~1.10.32]{chap:prob_tools_analysis_rand_opt_heurs}, we
    have that
    \begin{align*}
     \prob(X^{\ell} \ge (1+\ell^{-1/4}) \cdot \ell) \le\;& \exp \left( -\frac{\ell^{-1/2} (\ell-1)}{2(1+\ell^{-1/4})} \right)\\
      =\;& \exp(-\Omega(\ell^{1/2})). \qedhere
    \end{align*}
\end{proof}

We now derive a corresponding lower bound on the number of free riders
encountered during a period of the same length. The lemma below proves that we
gain at least $(1-\ell^{-1/4})\ell$ in fitness from free riders assuming $\ell$
improvements in a period, or the
\EA reaches the optimum during the period.

\begin{lemma}
    \label{lem:num_free_riders_lb}
    Consider a period with $\ell$ improving steps of the \EA optimising LO. Then, during this period, \wop, the \EA gains at least ${(1-\ell^{-1/4})\ell}$ in fitness from free
    riders in
    total, or the \EA reaches
    the optimum during the period.
\end{lemma}

\begin{proof}[Proof of Lemma~\ref{lem:num_free_riders_lb}]
    We define the random variables $X_{k}$ and $X^{\ell}$ as in the proof of
    Lemma~\ref{lem:num_free_riders_ub}. We may assume in the following that free riders are effectively drawn from an infinite bit
    string. This assumption is only false when an optimum is reached and this event is contained in the event whose probability we aim to bound from below.

    By
    \cite[Theorem~1.10.32]{chap:prob_tools_analysis_rand_opt_heurs} we
    have that
    \begin{align*}
     \prob(X^{\ell} \le (1-\ell^{-1/4}) \cdot \ell)
     \le\;& \exp \left( -\frac{\ell^{-1/2} \ell}{2-4\ell^{-1/4}/3} \right)\\
      =\;& \exp(-\Omega(\ell^{1/2})).\qedhere
    \end{align*}
\end{proof}

Combining
Lemmas~\ref{lem:1p1ea_lo_num_improvements_ub}~and~\ref{lem:num_free_riders_ub}
allows us to derive an upper bound on the total progress made by the algorithm
in a period of length $n^2/\psi$ which holds \wop

\begin{proof}[Proof of Lemma~\ref{lem:lo_progress_bounds}]
    \textit{Proof of Lemma~\ref{lem:lo_progress_bounds}~(\ref{lem:lo_progress_ub}):}\\
    Note that the progress of the algorithm stops abruptly if the global optimum is reached. Hence we assume pessimistically that progress is not bounded.
    Applying Lemma~\ref{lem:1p1ea_lo_num_improvements_ub} tells us that, \wop, the algorithm makes at least $\ell := (1+n^{-1/4}) \frac{\chi n}{\psi \cdot \exp \left( \frac{\chi j}{n} \right)}$ improvements. By Lemma~\ref{lem:num_free_riders_ub}, the number of free riders is at most
    $(1+\ell^{-1/4})\ell$ with probability $1-\exp(-\Omega(\ell^{1/2})) = 1 - \exp(-\Omega(n^{1/2}))$. Together, the fitness increases by at most
    \[
        (2+\ell^{-1/4})\ell =
        \frac{2\chi n}{\psi \cdot \exp \left( \frac{\chi j}{n}\right)} + o(n),
    \]
    \wop, as claimed.

    \textit{Proof of Lemma~\ref{lem:lo_progress_bounds}~(\ref{lem:lo_progress_lb}):}\\
    Lemma~\ref{lem:1p1ea_lo_num_improvements_lb} tells us that, with overwhelming
    probability, the algorithm makes at least
    $\ell := (1-n^{-1/4}) \frac{\chi n}{\psi \cdot \exp \left( \frac{\chi i}{n-\chi} \right)}$
    improvements within a period of
    length $n^2/\psi$ or exceeds some fitness $i$ as defined in the statement
    of this theorem. If the algorithm exceeds
    fitness
    $i$ at the end of the period then we are done, since this event is contained in the event whose probability we aim to bound from below. Assuming at least $\ell$ improvements are made, by Lemma~\ref{lem:num_free_riders_lb} the total gain through free riders is at least $(1-\ell^{-1/4})\ell$ with probability $1-\exp(-\Omega(\ell^{1/2})) = 1-\exp(-\Omega(n^{1/2}))$. Hence, \wop the \EA makes at least
    \begin{align*}
     &(2-\ell^{-1/4})\ell = \frac{2\chi n}{\psi \cdot \exp \left( \frac{\chi i}{n-\chi} \right)} - o(n)
    \end{align*}
    progress in total during the period, or the \EA exceeds a fitness of $i$ at the end of the period.

    Finally, we argue that the term $n-\chi$ in the $\exp$-term can be replaced
    by~$n$ since
    \begin{align*}
        \exp\left(\frac{\chi i}{n-\chi}\right)
        =\;& \exp\left(\frac{\chi i}{n} \cdot \frac{n}{n-\chi}\right)\\
        =\;& \exp\left(\frac{\chi i}{n} + \frac{\chi^2i}{n(n-\chi)}\right)\\
        =\;& \exp\left(\frac{\chi i}{n}\right) \cdot \exp\left(\frac{\chi^2i}{n(n-\chi)}\right)\\
        \le\;& \exp\left(\frac{\chi i}{n}\right) \cdot \exp\left(\frac{\chi^2}{n-\chi}\right)\\
        \le\;& \exp\left(\frac{\chi i}{n}\right) \cdot \frac{1}{1-\frac{\chi^2}{n-\chi}}
    \end{align*}
    where in the last step we used $e^x \le \frac{1}{1-x}$ for $x < 1$.
    Together,
    \begin{align*}
    &  \frac{2\chi n}{\psi \cdot \exp \left( \frac{\chi i}{n-\chi} \right)} - o(n)\\
    & \ge  \frac{2\chi n}{\psi \cdot \exp \left( \frac{\chi i}{n} \right)} \left(1-\frac{\chi^2}{n-\chi}\right)  - o(n)\\
    & = \frac{2\chi n}{\psi \cdot \exp \left( \frac{\chi i}{n} \right)} - O(1) - o(n)
    \end{align*}
    and the $O(1)$ term is absorbed in the $o(n)$ term.
\end{proof}

\subsection{Proof of Lemma~\ref{lem:lo_fitness_bounds}}

Since Lemma~\ref{lem:lo_progress_bounds} provides upper and lower bounds on the
total progress of the algorithm in a period of length $n^2/\psi$, we are now
able to finally prove Lemma~\ref{lem:lo_fitness_bounds}.

\begin{proof}[Proof of Lemma~\ref{lem:lo_fitness_bounds}]
    We prove the statement by induction. If an optimum has been reached at the end of period~$i$ then there is nothing to prove. We may therefore assume that $\ell_{\chi , i} < n$.
     With overwhelming probability, the initial fitness will be in $[\ell_{\chi ,0},u_{\chi ,0}]:=[0,\sqrt{n}]$, since the fitness of the initial search point
    is
    at most $\sqrt{n}$ with probability $1-2^{-\sqrt{n}}$.

    By Lemma~\ref{lem:lo_progress_bounds}~(\ref{lem:lo_progress_ub}) we have that if the
    individual begins with a fitness of $u_{\chi,i}$ then, with
    overwhelming probability, after a period of $n^2/\psi$ iterations it has a
    fitness of at most
    \[ u_{\chi,i+1} := u_{\chi,i} +  \frac{2\chi n}{\psi \exp \left( \frac{\chi  u_{\chi,i}}{n} \right)} + o(n). \]
    Note that the upper bound $u_{\chi , i+1}$ still holds (trivially) if $u_{\chi , i+1} \ge n$.

    We now apply Lemma~\ref{lem:lo_progress_bounds}~(\ref{lem:lo_progress_lb}) to show that the fitness of the
    individual at the end of period $i+1$ is at least $\ell_{\chi,i+1}$ given that
    its fitness at the end of period $i$ is at least $\ell_{\chi,i}$.
    Lemma~\ref{lem:lo_progress_bounds}~(\ref{lem:lo_progress_lb}) tells us that, with overwhelming
    probability, the fitness of the individual in the \EA is at least
    \[ \ell_{\chi , i+1} := \ell_{\chi,i} + \frac{2\chi n}{\psi \cdot \exp \left( \frac{\chi  u_{\chi,i+1}}{n} \right)} - o(n) \]
    or it exceeds $u_{\chi,i+1}$. Since we do not exceed fitness $u_{\chi,i+1}$ with
    overwhelming probability (and we certainly do not exceed it if $u_{\chi , i+1} \ge n$), we have that the fitness is at least $\ell_{\chi , i+1}$, \wop

    Since at the beginning of period $i+1$ the individual has a fitness in the
    interval $[\ell_{\chi,i},u_{\chi,i}]$ and at the end of the period it has a
    fitness in the interval $[\ell_{\chi,i}, u_{\chi,i+1}]$, and since both
    $\ell_{\chi ,j}$ and $u_{\chi ,j}$ are monotonically increasing for all $j$ we can
    conclude that, \wop, the fitness of the individual
    remains in the interval $[\ell_{\chi,i},u_{\chi,i+1}]$ throughout the period.
    Taking a union bound over all failure probabilities concludes the proof of
    both claims.
\end{proof}

\subsection{Proof of Lemma~\ref{lem:recurrence_defs}}

This proof uses the fitness bounds derived in Lemma~\ref{lem:lo_fitness_bounds}.

\begin{proof}[Proof of Lemma~\ref{lem:recurrence_defs}]
    The statement about $c_{\ell, \chi, 0}$ and $c_{u, \chi, 0}$ is obvious as
    $\ell_{\chi , 0} = 0$ and $u_{\chi , 0} = \sqrt{n}$.

    By definition, $u_{\chi,i+1}$ can be written as
    \begin{align*}
     & c_{u,\chi,i} \cdot n + \frac{2\chi n}{\psi \exp \left( \frac{\chi  \cdot (c_{u,\chi,i} \cdot n + o(n))}{n} \right)} + o(n)\\
     \le\;&
    c_{u,\chi,i} \cdot n + \frac{2\chi n}{\psi \exp \left( \chi \cdot c_{u,\chi,i} \right)} + o(n)
    \end{align*}
    and the leading constant is
    \[ c_{u,\chi,i+1} := c_{u,\chi,i} + \frac{2\chi}{\psi \exp \left( \chi \cdot c_{u,\chi,i} \right)}. \]

    By definition, $\ell_{\chi , i+1}$ can be written as
    \begin{align*}
    & c_{\ell,\chi,i} \cdot n + \frac{2\chi n}{\psi \exp \left( \frac{\chi (u_{\chi,i+1}+o(n))}{n} \right)} - o(n)\\
    =\;& c_{\ell,\chi,i} \cdot n + \frac{2\chi n}{\psi \exp \left( \chi \cdot c_{u,\chi,i+1}+o(1) \right)} - o(n).
    \end{align*}
    This is at least
    \[ c_{\ell,\chi,i} \cdot n + \frac{2\chi n}{\psi \exp \left( \chi \cdot c_{u,\chi,i+1} \right)} - o(n) \]
    since $e^{-o(1)} \ge 1-o(1)$ since $e^{-x} \le 1 - x/2$ for ${0 \le x \le
    1}$ \cite[Lemma~1.4.2]{chap:prob_tools_analysis_rand_opt_heurs}. The leading constant is thus
    \[
        c_{\ell,\chi,i+1} := c_{\ell,\chi,i} + \frac{2\chi}{\psi \exp \left( \chi \cdot c_{u,\chi,i+1} \right)}. \qedhere
    \]
\end{proof}

\subsection{Proof of Lemma~\ref{lem:lo_param_space_unimodal_concave}}

\begin{proof}[Proof of Lemma~\ref{lem:lo_param_space_unimodal_concave}]
    In order to bound the fitness of the individual in the \EA after $\alpha
    n^2$ we simply iterate the recurrences given by
    Lemma~\ref{lem:recurrence_defs} ${\psi \cdot \alpha}$ times. We do so for all
    cutoff times in the set \{$0.000001n^2, 0.000002n^2, \ldots, 772074n^2,
    772075n^2$\}, setting $\psi = 1000000$.  For all configurations and cutoff times
    which we consider here, the upper bound on the leading constant of the fitness is strictly less than $1$. Then Lemma~\ref{lem:recurrence_defs} implies
    that, \wop, no configuration reaches the optimum
    within any of these cutoff times, and hence we can ignore the case in
    Lemma~\ref{lem:lo_fitness_bounds} that the optimum is reached by the end of
    a period and simply assume that the fitness is contained in the interval
    given by the lemma. By Lemma~\ref{lem:recurrence_defs}, the fitness of the
    individual in the \EA is in the range $[c_{\ell,\chi,i}n - o(n),c_{u,\chi,i+1}n +
    o(n)]$ \wop Hence, for two parameters $a, b$, if
    the interval $[c_{\ell,a,i},c_{u,a,i+1}]$ is non-overlapping with the
    interval $[c_{\ell,b,i},c_{u,b,i+1}]$ and $c_{\ell,a,i} > c_{u,b,i+1}$ then
    we can conclude that, for all times $t$ satisfying $\tau_{i} \le t \le
    \tau_{i+1}$, where $\tau_{i}$ is the number of iterations corresponding to
    the end of period $i$, the \EAc[a] is ahead of the \EAc[b] with
    overwhelming probability. \par

    We conducted the above verification for all pairs of neighbouring
    configurations (i.e. for all configurations between which it is possible to
    transition in a single mutation using the local search operator) for all
    cutoff times up to $0.772075n^2$. We discovered that each value of $\chi
    \ge 1.6$ is optimal for some range of quadratic cutoff times (bounds on which are given in Table~\ref{tab:param_space_monotonic_ranges}), and also that for these ranges of cutoff times the parameter landscape is unimodal.
\end{proof}

\subsection{Proof of Lemma~\ref{lem:remains_ahead_ineq}}

\begin{proof}[Proof of Lemma~\ref{lem:remains_ahead_ineq}]
    We know by Lemma~\ref{lem:lo_progress_bounds}~(\ref{lem:lo_progress_lb}) that, with overwhelming
    probability, for any $i$ greater than the current fitness, the \EA
    makes progress of at least
    \[ \frac{2\chi n}{\psi \cdot \exp \left( \frac{\chi i}{n-\chi} \right)} - o(n) \]
    in a period of length $n^2/\psi$ or exceeds a fitness of $i$ at the end of
    the period. We know by Lemma~\ref{lem:lo_progress_bounds}~(\ref{lem:lo_progress_ub}) that, with overwhelming
    probability, the \EA makes progress of at most
    \[ \frac{2\chi n}{\psi \cdot \exp \left( \frac{\chi j}{n} \right) } + o(n) \]
    in a period of length $n^2/\psi$, given that it begins the period at
    fitness~$j$. \par

    Setting $i$ in Lemma~\ref{lem:lo_progress_bounds}~(\ref{lem:lo_progress_lb}) to $n-1$ tells us that, with
    overwhelming probability, the \EAc[a] either makes progress of at least
    \begin{align*}
        \frac{2an}{\psi \cdot \exp \left( \frac{a(n-1)}{n-a} \right)} - o(n) &\ge \frac{2an}{\psi \cdot \exp \left( a \right)} - o(n)
    \end{align*}
    in a period of length $n^2/\psi$ (where the inequality holds by the same reasoning as in the proof of Lemma~\ref{lem:recurrence_defs}) or exceeds a fitness of $n-1$ at
    the end of the period (i.e.\ reaches the optimum). The quantity $n - (c_{\ell,a,i} \cdot n - o(n))$ is an upper bound on the distance to the optimum of the \EAc[a] and therefore if the algorithm makes more progress than this then it has reached the optimum. We have that
    \begin{align*}
     \frac{2an}{\psi \cdot \exp \left( a \right)} - o(n) \ge\;& n - (c_{\ell,a,i} \cdot n - o(n)) \\
     \iff \frac{2a}{\psi \cdot \exp \left( a \right)} \ge\;& 1 - c_{\ell,a,i} + o(1)\\
    \iff \psi \le\;& \frac{2a}{(1 - c_{\ell,a,i} + o(1)) \cdot \exp \left( a \right)}
    \end{align*}
    for which a sufficient condition is
    \[ \psi \le \frac{2a}{(1 - c_{\ell,a,i}) \cdot \exp \left( a \right)}. \]
    Recalling that the length of a period is $n^2/\psi$, this tells us that, when beginning at a fitness of at least $c_{\ell,a,i} \cdot n - o(n)$, the
    \EAc[a] reaches the optimum \wop for all periods
    of length at least
    \begin{equation}
        \label{eqn:max_time_a_opt}
        \frac{n^2}{\left( \frac{2a}{(1 - c_{\ell,a,i}) \cdot \exp \left( a \right)} \right)}
    \end{equation}

    Let us now derive a lower bound on the time required for the \EAc[b] to
    make progress of at least the initial distance between the \EAc[a] and itself. Since by assumption the fitness of the
    \EAc[a] is at least $c_{\ell,a,i} \cdot n-o(n)$ and the fitness of the
    \EAc[b] is at most $c_{u,b,i} \cdot n+o(n)$ we have a lower bound on this distance of
    $(c_{\ell,a,i}-c_{u,b,i}) \cdot n - o(n)$. By Lemma~\ref{lem:lo_fitness_bounds} we have that,
    \wop, the \EAc[b] makes progress of at most
    \[ \frac{2bn}{\psi \cdot \exp \left( \frac{b(c_{\ell,b,i} \cdot n-o(n))}{n} \right)} + o(n) \le \frac{2bn}{\psi \cdot \exp \left( b c_{\ell,b,i} \right)} + o(n) \]
    in a period of length $n^2/\psi$, where the inequality holds due to
    \cite[Lemma~1.4.2(a)]{chap:prob_tools_analysis_rand_opt_heurs}.
    Therefore the \EAc[b] does not cover the initial distance between the two
    algorithms if
    \[ \frac{2 bn}{\psi \cdot \exp \left( bc_{\ell,b,i} \right)} + o(n) \le (c_{\ell,a,i}-c_{u,b,i}) \cdot n - o(n) \]
    \[ \iff \psi \ge \frac{2 b}{(c_{\ell,a,i}-c_{u,b,i}) \cdot \exp \left( bc_{\ell,b,i} \right)} + o(1) \]
    for which a sufficient condition is
    \[ \psi \ge \frac{2 b}{(c_{\ell,a,i}-c_{u,b,i}) \cdot \exp \left( bc_{\ell,b,i} \right)} + \varepsilon \]
    for some positive constant $\varepsilon$.  Recalling that the length of the
    period is $n^2/\psi$, we see that, with overwhelming probability, the
    \EAc[b] requires at least
    \begin{equation}
        \label{eqn:min_time_b_covers_D}
        \frac{n^2}{\left( \frac{2 b}{(c_{\ell,a,i}-c_{u,b,i}) \cdot \exp \left( bc_{\ell,b,i} \right)} + \varepsilon \right)}
    \end{equation}
    iterations before the probability that it has covered the initial distance
    between the two algorithms is not overwhelmingly small.  Combining the
    bounds on $\psi$ in
    Equations~\eqref{eqn:max_time_a_opt}~and~\eqref{eqn:min_time_b_covers_D} we
    conclude that the \EAc[a] reaches the optimum of LO
    before the \EAc[b] catches up if
    \[ \frac{n^2}{\left( \frac{2 b}{(c_{\ell,a,i}-c_{u,b,i}) \cdot \exp \left( bc_{\ell,b,i} \right)} + \varepsilon \right)} \ge \frac{n^2}{\left( \frac{2a}{(1 - c_{\ell,a,i}) \cdot \exp \left( a \right)} \right)} \]
    which holds if and only if
    \[ \frac{\left( \frac{2 b}{(c_{\ell,a,i}-c_{u,b,i}) \cdot \exp \left( bc_{\ell,b,i} \right)} + \varepsilon \right)}{\left( \frac{2a}{(1 - c_{\ell,a,i}) \cdot \exp \left( a \right)} \right)} \le 1 \qedhere
    \]
\end{proof}

\subsection{Proof of Lemma~\ref{lem:lo_param_landscape_monotone}}

\begin{proof}[Proof of Lemma~\ref{lem:lo_param_landscape_monotone}]
    We split this proof into two parts. In the first part, we prove that, from
    time $0.772076n^2$ until some time $t_a$ the \EAc[a] remains
    ahead of the \EAc[b] \wop In the second part, we
    prove that for cutoff times larger than $t_a$ the \EAc[a] will find the
    optimum before the \EAc[b] catches up with where it began the period, with
    overwhelming probability. These two cases together imply that, with
    overwhelming probability, the \EAc[a] wins a ParamRLS-F comparison against
    the \EAc[b] for all cutoff times $\kappa \ge 0.772076n^2$. \par

    We first verify that the \EAc[a] is ahead of the \EAc[b] at all times
    between $0.772076n^2$ and some time~$t_a$. Lemma~\ref{lem:recurrence_defs}
    tells us that, from the end of period $i$ to the end of period $i+1$, the
    leading constant of the fitness of the individual in the
    \EA is in the range $[c_{\ell,\chi,i},c_{u,\chi,i+1}]$ with overwhelming
    probability, where $c_{\ell,\chi,i}$ and $c_{u,\chi,i+1}$ are as defined in
    Lemma~\ref{lem:recurrence_defs}. Hence if we verify for all periods $i$
    satisfying $772076 \le i \le \psi \cdot t_a$ that the intervals
    $[c_{\ell,a,i},c_{u,a,i+1}]$ and $[c_{\ell,b,i},c_{u,b,i+1}]$ are
    non-overlapping then we can conclude that, \wop,
    the \EAc[a] is ahead of the \EAc[b] for all times $t$ satisfying
    $0.772076n^2 \le t \le t_a$. For each $b$, the time $t_a$ is chosen to be
    the end of the final period for which $c_{u,b,i}<1$.  For these values of
    $t_a$, we verified that the above intervals are non-overlapping and hence
    that the \EAc[a] is ahead of each worse configuration for all time $t$
    satisfying $0.772076n^2 \le t \le t_a$. \par

    If $a \le 1.6$ and for all times $t$ satisfying $0.772076n^2 \le t \le
    t_a$ the condition in Lemma~\ref{lem:remains_ahead_ineq}
    holds for all $b < a$ then we can conclude that, with overwhelming
    probability, the \EAc[a] reaches the optimum before any \EAc[b] catches it.
    Therefore for any cutoff time $\kappa \ge t_a$ the \EAc[a] wins a
    comparison in ParamRLS-F against any \EAc[b] for $b > a$.  Similarly, if $a
    \ge 1.6$ and the condition in Lemma~\ref{lem:remains_ahead_ineq} holds for
    all $b > a$ then we can conclude that \wop the
    \EAc[a] reaches the optimum before any \EAc[b] catches it. Therefore for
    any cutoff time $\kappa \ge t_a$ the \EAc[a] wins a comparison in
    ParamRLS-F against any \EAc[b] for $b > a$.  \par

    We can therefore prove the claim by verifying that, for all $a \le 1.6$ and
    $b < a$ the inequality in Lemma~\ref{lem:remains_ahead_ineq} holds, and
    also for all $a \ge 1.6$ and $b > a$ the same inequality is true. We do so
    for the specific tuning scenario given by this theorem by iterating over
    all cases where we require this inequality to hold and verifying that it
    does so in each case. We used a value of $\varepsilon=0.00000000001$. Table~\ref{tab:vals_of_lemma_ineq} contains the
    values of the quantity from Lemma~\ref{lem:remains_ahead_ineq}, scaled up by a factor of $100000$ for readability, for all
    required comparisons for values of $\chi \le 1.6$. It is easily verified that
    all values are much smaller than $100000$, as required. This proves our claim.
\end{proof}


\begin{table*}
    \small
    \centering
    \begin{tabular}{|c| r@{\hspace{0.6em}}r@{\hspace{0.6em}}r@{\hspace{0.6em}}r@{\hspace{0.6em} }r@{\hspace{0.6em} }r@{\hspace{0.6em}}r@{\hspace{0.6em}}r@{\hspace{0.6em}}r@{\hspace{0.6em} }r@{\hspace{0.6em} }r@{\hspace{0.6em}}r@{\hspace{0.6em}}r@{\hspace{0.6em}}r@{\hspace{0.6em} }r@{\hspace{0.6em} }|}
    \hline
            & 0.1 & 0.2 & 0.3 & 0.4 & 0.5 & 0.6 & 0.7 & 0.8 & 0.9 & 1.0 & 1.1 & 1.2 & 1.3 & 1.4 & 1.5 \\
    \hline
        0.2 & 0.0 & -- & -- & -- & -- & -- & -- & -- & -- & -- & -- & -- & -- & -- & -- \\
        0.3 & 0.0 & 0.2 & -- & -- & -- & -- & -- & -- & -- & -- & -- & -- & -- & -- & -- \\
        0.4 & 0.0 & 0.1 & 0.4 & -- & -- & -- & -- & -- & -- & -- & -- & -- & -- & -- & -- \\
        0.5 & 0.0 & 0.1 & 0.2 & 0.4 & -- & -- & -- & -- & -- & -- & -- & -- & -- & -- & -- \\
        0.6 & 0.0 & 0.1 & 0.2 & 0.3 & 0.8 & -- & -- & -- & -- & -- & -- & -- & -- & -- & -- \\
        0.7 & 0.0 & 0.1 & 0.1 & 0.3 & 0.5 & 1.1 & -- & -- & -- & -- & -- & -- & -- & -- & -- \\
        0.8 & 0.0 & 0.1 & 0.2 & 0.3 & 0.5 & 0.9 & 2.2 & -- & -- & -- & -- & -- & -- & -- & -- \\
        0.9 & 0.0 & 0.1 & 0.2 & 0.3 & 0.4 & 0.7 & 1.2 & 2.9 & -- & -- & -- & -- & -- & -- & -- \\
        1.0 & 0.1 & 0.1 & 0.2 & 0.3 & 0.5 & 0.8 & 1.3 & 2.2 & 5.2 & -- & -- & -- & -- & -- & -- \\
        1.1 & 0.1 & 0.1 & 0.2 & 0.3 & 0.5 & 0.8 & 1.1 & 1.8 & 3.1 & 7.3 & -- & -- & -- & -- & -- \\
        1.2 & 0.1 & 0.2 & 0.3 & 0.4 & 0.6 & 0.9 & 1.3 & 2.0 & 3.1 & 5.5 & 12.9 & -- & -- & -- & -- \\
        1.3 & 0.1 & 0.1 & 0.3 & 0.4 & 0.6 & 0.8 & 1.2 & 1.7 & 2.6 & 4.1 & 7.3 & 17.5 & -- & -- & -- \\
        1.4 & 0.1 & 0.2 & 0.3 & 0.4 & 0.7 & 0.9 & 1.3 & 1.8 & 2.7 & 4.0 & 6.5 & 12.0 & 30.1 & -- & -- \\
        1.5 & 0.1 & 0.2 & 0.4 & 0.6 & 0.9 & 1.3 & 1.8 & 2.5 & 3.6 & 5.3 & 8.3 & 14.0 & 27.8 & 78.0 & -- \\
        1.6 & 0.1 & 0.3 & 0.4 & 0.7 & 1.0 & 1.4 & 2.0 & 2.7 & 3.9 & 5.7 & 8.8 & 14.5 & 27.4 & 65.9 & 294.9 \\
        \hline
            & 1.7 & 1.8 & 1.9 & 2.0 & 2.1 & 2.2 & 2.3 & 2.4 & 2.5 & 2.6 & 2.7 & 2.8 & 2.9 & 3.0 \\
        \hline
        1.6 & 245.3 & 67.0 & 31.3 & 18.3 & 12.1 & 8.7 & 6.6 & 5.2 & 4.2 & 3.5 & 2.9 & 2.5 & 2.2 & 1.9 & \\
        1.7 & -- & 95.3 & 37.1 & 20.5 & 13.2 & 9.3 & 7.0 & 5.5 & 4.4 & 3.6 & 3.1 & 2.6 & 2.3 & 2.0 & \\
        1.8 & -- & -- & 69.8 & 29.9 & 17.6 & 11.8 & 8.6 & 6.6 & 5.3 & 4.3 & 3.6 & 3.1 & 2.7 & 2.4 & \\
        1.9 & -- & -- & -- & 52.0 & 23.3 & 14.2 & 9.8 & 7.3 & 5.7 & 4.6 & 3.8 & 3.2 & 2.8 & 2.4 & \\
        2.0 & -- & -- & -- & -- & 46.6 & 21.4 & 13.3 & 9.3 & 7.0 & 5.6 & 4.5 & 3.8 & 3.2 & 2.8 & \\
        2.1 & -- & -- & -- & -- & -- & 49.0 & 22.9 & 14.4 & 10.2 & 7.8 & 6.2 & 5.1 & 4.3 & 3.7 & \\
        2.2 & -- & -- & -- & -- & -- & -- & 43.4 & 20.5 & 13.0 & 9.3 & 7.2 & 5.7 & 4.7 & 4.0 & \\
        2.3 & -- & -- & -- & -- & -- & -- & -- & 42.8 & 20.4 & 13.0 & 9.4 & 7.2 & 5.8 & 4.8 & \\
        2.4 & -- & -- & -- & -- & -- & -- & -- & -- & 42.2 & 20.3 & 13.0 & 9.4 & 7.3 & 5.9 & \\
        2.5 & -- & -- & -- & -- & -- & -- & -- & -- & -- & 42.6 & 20.5 & 13.3 & 9.6 & 7.5 & \\
        2.6 & -- & -- & -- & -- & -- & -- & -- & -- & -- & -- & 44.8 & 21.7 & 14.1 & 10.3 & \\
        2.7 & -- & -- & -- & -- & -- & -- & -- & -- & -- & -- & -- & 43.6 & 21.2 & 13.8 & \\
        2.8 & -- & -- & -- & -- & -- & -- & -- & -- & -- & -- & -- & -- & 45.9 & 22.4 & \\
        2.9 & -- & -- & -- & -- & -- & -- & -- & -- & -- & -- & -- & -- & -- & 46.0 & \\
        \hline
    \end{tabular}
    \caption{The number in row $a$ column $b$ is $100000$ times the value of the
    quantity given in Lemma~\ref{lem:remains_ahead_ineq} (to one decimal
    place) for the \EAc[a] ahead of the \EAc[b] by some linear distance. Hence
    if it is no greater than $100000$ then, \wop, the
    \EAc[a] remains ahead of the \EAc[b]. The values have been displayed in
    this way to give an idea of their relative size, since all values are so
    small this relationship was otherwise lost when reducing the size of the
    table. It is easily verified that all values are several orders of magnitude
    smaller than we require.}
    \label{tab:vals_of_lemma_ineq}
\end{table*}

\end{document}